\documentclass{article}

\usepackage[final]{neurips_2020}

\usepackage[utf8]{inputenc} %
\usepackage[T1]{fontenc}    %
\usepackage{hyperref}       %
\usepackage{url}            %
\usepackage{booktabs}       %
\usepackage{amsfonts}       %
\usepackage{nicefrac}       %
\usepackage{microtype}      %
\usepackage{qiangstyle}
\usepackage{mathrsfs}
\usepackage{tabularx}
\usepackage{wrapfig}
\usepackage{microtype}
\usepackage{graphicx}
\usepackage{subfigure}
\usepackage{booktabs}
\usepackage{soul}

\title{Off-Policy Interval Estimation  with \\ Lipschitz Value Iteration}

\author{%
 Ziyang Tang\\
  University of Texas at Austin\\
  \texttt{ztang@utexas.edu} \\
  \And
  Yihao Feng \\
  University of Texas at Austin\\
  \texttt{yihao@cs.utexas.edu} \\
  \AND
  Na Zhang\\
  Tsinghua University\\
  \texttt{zhangna@pbcsf.tsinghua.edu.cn}\\
  \And
  Jian Peng\\
  University of Illinois at Urbana-Champaign\\
  \texttt{jianpeng@illinois.edu}\\
  \And
  Qiang Liu\\
  University of Texas at Austin\\
  \texttt{lqiang@cs.utexas.edu} \\
}

\usepackage{lipsum} 

\begin{document}
\maketitle

\begin{abstract}
Off-policy evaluation provides an essential tool for evaluating the effects of different policies or treatments using only observed data. When applied to high-stakes scenarios such as medical diagnosis or financial decision-making, it is crucial to provide provably correct upper and lower bounds of the expected reward, not just a classical single point estimate, to the end-users, as executing a poor policy can be very costly. In this work, we propose a provably correct method for obtaining interval bounds for off-policy evaluation in a general continuous setting. The idea is to search for the maximum and minimum values of the expected reward among all the Lipschitz Q-functions that are consistent with the observations, which amounts to solving a constrained optimization problem on a Lipschitz function space. We go on to introduce a Lipschitz value iteration method to monotonically tighten the interval, which is simple yet efficient and provably convergent. We demonstrate the practical efficiency of our method on a range of benchmarks.
\end{abstract}

\section{Introduction}
Reinforcement learning (RL) \citep[e.g.,][]{sutton98beinforcement} \zt{has become} widely used in tasks like recommendation system, robotics, trading and healthcare \citep{murphy01marginal, li11unbiased, bottou13counterfactual, thomas17predictive}.
\zt {The current success of RL highly relies on excessive amount of data, which, however, is usually not available in many real world tasks where deploying a new policy is very costly or even risky.
} %
Off-policy evaluation (OPE) \citep[e.g.,][]{fonteneau13batch, jiang16doubly,liu2018breaking}, estimating the expected reward of a target policy using observational data gathered from previous behavior policies, therefore holds tremendous promise for designing data-efficient RL algorithms by leveraging on previously collected data.

Existing OPE methods mainly focus on \textit{point estimation}, which only provides a single point estimation of the expected reward. However, such point estimate can be rather unreliable as OPE often suffers from high error due to the lack of historical samples, policy shift or model misspecification. 
Further, for applications in high-stakes areas such as medical diagnosis and financial investment, a point estimate itself is far from enough and can even be dangerous if it is unreliable. Hence, it is essential to provide \textit{provably correct interval estimation} of the expected reward, which is both trustful and theoretically correct.

To address this problem, we propose a general optimization-based framework to derive a provably correct off-policy interval estimation based on historical samples.
Our idea is to search for the largest and smallest possible values of the expected reward, among all the Q-functions in a Lipschitz function space that are consistent with the observed historical samples.
This interval estimator is provably correct once the true Q-function satisfies the Lipschitz assumption.

Computing our upper and lower bounds amounts to solving a constrained optimization problem \zt{in} the space of Lipschitz functions.
\zt{We introduce a \emph{Lipschitz value iteration} algorithm of a similar style to value iteration.}
\zt{Our method} is efficient and provably convergent. 
In particular, our algorithm has a simple closed form update at each iteration and is guaranteed to monotonically tighten the bounds with a linear rate under mild conditions. 
To speed up the algorithm, we develop a double subsampling strategy,  \zt{which we only pick a random subsample of value functions to update in each iteration and use the same batch of data as constraints.}

We test our algorithm on a number of benchmarks and show that it can provide tight and provably correct bounds.

\noindent\textbf{Related Work}\quad
Our work is closely related to the off-policy point estimation.
There are typically two types of OPE methods, importance sampling (IS) based methods \citep[e.g.,][]{liu01monte, precup00eligibility, liu2018breaking, xie2019optimal} and value function or model based methods \citep[e.g.,][]{fonteneau13batch, liu2018representation, le2019batch, feng2019kernel}
Another line of work combines these two methods and yields a doubly-robust estimator for off-policy evaluation \citep{jiang16doubly, thomas2016data, kallus2019efficiently, tang2020doubly}.
In this work, we \zt{consider} the black box setting \zt{when the behavior policy is assumed to be unknown} \citep[e.g.,][]{
nachum2019dualdice,
mousavi2020blackbox,zhang2020gendice,fengaccountable2020}.

\zt{IS-based} point estimation methods 
naturally yield a confidence interval by standard concentration \citep{thomas2015high, thomas2015bhigh}.
\zt{Another major type of confidence interval estimation approaches} leverages the statistical procedure of bootstrapping \citep{white2010interval, hanna2017bootstrapping}.
However, these confidence intervals are typically loose due to the curse of horizon. Moreover, they heavily rely on the assumption that the off-policy data is drawn i.i.d. from a particular behavior policy. But this is \zt{not} always true since the policies usually evolve and depend on their previous policies.

\zt{Another related set of works are} PAC-RL \citep[e.g.,][]{jin2018q, dann2018policy, song2019efficient, yang2019learning}, which mainly focus on the regret bound or sample complexity of the Q-learning exploration.
As a side product, they also provide a confidence interval estimation for the value function.
However, this line of work mostly focuses on the tabular or linear MDPs. In contrast, our work aims to handle the general continuous MDPs by leveraging 
\zt{Lipschitz properties of Q-functions}.
\citet{song2019efficient} provides a metric embedding style Q-learning method, \zt{with a particular} focus on finite horizon settings.

\section{Background and Problem Settings}

We firstly set up the problem of 
off-policy interval evaluation and then introduce the related background on Q-learning and Bellman equation. 

\textbf{Markov Decision Process}\quad
A Markov decision process (MDP) 
$M = \langle \Sset, \Aset, r, \T, \mu_0, \gamma\rangle$ 
consists of a state space $\Sset$, an action space $\Aset$, an unknown deterministic reward function $r:\Sset\times \Aset \to \RR$, an unknown transition function $\T:\Sset\times \Aset \to \Sset$, an initial state distribution $\mu_0$, and a discounted factor $\gamma$.  
Throughout this work, 
we assume that the transition and reward function are deterministic for simplicity and we can easily draw samples from the initial state distribution $\mu_0$.

\zt{In RL,} an agent acts in a MDP following a policy $\pi(\cdot |s)$, 
which prescribes a distribution over the action space $\Aset$ given each state $s \in \Sset$.
Running the policy starting from the initial distribution $\mu_0$ yields a random trajectory $\tau := \{s_i,a_i,r_i\}_{1\leq i\leq T}$,
 where $s_i, a_i, r_i$ represent the state, action, reward at time $i$ respectively.  
We define the infinite horizon discounted reward of $\pi$ as 
$
\Rpi := \lim_{T\to \infty} \E_{\tau \sim \pi}
\left [{\sum_{i=0}^T \gamma^i r_i}\right]\,,
$
where $\gamma \in (0,1)$ is a discounting factor; $T$ 
is the horizon length of the trajectory, which we assume to approach infinite, hence yielding an \emph{infinite horizon problem}; 
$\E_{\tau \sim \pi}[\cdot]$ denotes the expectation of the random trajectories collected under the policy $\pi$. 

\textbf{Black Box Off-Policy Interval Estimation}\quad
We are interested in the problem of \emph{black-box off-policy  interval evaluation}, \zt{which requires arguably the minimum assumptions on the off-policy data}.
It amounts to providing an interval estimation $[\underline{\Rpi}, \overline{\Rpi}]$ of the expected reward $R^\pi$ of a policy $\pi$ (called the target policy), given a set of transition pairs 
$\{s_i,a_i,s_i',r_i\}_{i=1}^n$ collected under a \emph{different, unknown} behavior policy, or even a mix of different policies; 
here $s_i' = \T(s_i,a_i)$ and  $r_i = r(s_i,a_i)$ 
denote the next state and the local reward following $(s_i,a_i)$ respectively.

\textbf{Q-function and Bellman Equation}\quad 
We review the properties of the Q-function that is most relevant to our work.
The Q-function $Q^{\pi}(s,a)$ %
specifies the expected reward when following $\pi$ from the state-action pair $(s,a)$
and is known to be the unique fixed point of the Bellman equation: 
\begin{align}\label{equ:ExactBellman} 
Q(s,a) = \mathcal B^\pi Q(s,a) := r(s,a) + \gamma \Ppi Q(s,a), ~~~~\forall (s,a)\in \Sset \times \Aset\,,
\end{align} 
where $\mathcal B^\pi$ denotes the Bellman operator, 
and $\Ppi$ is the transition operator defined by 
\begin{equation}
\label{eqn:transition_op}
    \Ppi Q(s,a) := \E_{s'=\T(s,a), a'\sim \pi(\cdot|s')}[Q(s',a')]\,.
\end{equation}
An expected reward associated with a 
Q-function is defined via  %
\begin{align} \label{equ:defRQ} 
\zt{R_{\mu_{0,\pi}}}[Q]:= %
\E_{s_0, a_0 \sim \mu_{0,\pi}}\left[Q(s_0,a_0)\right]\,,
\end{align}
where we use $\mu_{0,\pi}(s,a) = \mu_0(s)\pi(a|s)$ to denote the joint initial state-action distribution.

\textbf{Off-Policy Q-Learning}\quad
Q-function can be learned in an off-policy manner. 
Assume we have a set of transition pairs $\mathcal D:= \{s_i, a_i, s_i, r_i'\}_{i=1}^n$. 
Under the assumption of deterministic transition and reward, we can estimate the Bellman operator on each of the data point: 
$$
\B^{\pi} Q(s_i,a_i) = r_i + \E_{a'\sim \pi(\cdot|s'_i)}[ Q(s_i', a') ]\,, 
$$
which can be estimated with an arbitrarily high accuracy via drawing a large number of samples from $\pi(\cdot|s_i')$. 
Assume $Q^\pi$ belongs to a class of functions $\mathcal F$, we can estimate $Q^\pi$ by finding a $Q \in \F$ that satisfies the Bellman equation on the data points: 
\begin{align} \label{equ:Bellmani}
Q \in \mathcal F, ~~~s.t.~~~Q(s_i,a_i) = \B^\pi Q(s_i,a_i), ~~~ \forall i\in [n]. 
\end{align}
Compared with the exact Bellman equation \eqref{equ:ExactBellman}, 
we only match the equation on the observed data \eqref{equ:Bellmani}, which may yield multiple or even infinite solutions of $Q$.  
Therefore, the function class $\F$ needs to be sufficiently constrained in order to yield meaningful solutions.  
In practice, \eqref{equ:Bellmani} is often solved using fitted value iteration \citep{munos2008finite}, which starts from an initial $Q_0$, and then perform iterative updates by 
\begin{align}\label{equ:fittedQ}
Q_{t+1} \gets \argmin_{Q\in \mathcal F}
\sum_{i=1}^n\left(Q(s_i,a_i) - \mathcal B^\pi Q_{t}(s_i, a_i)\right)^2. 
\end{align}
It is easy to see $R^\pi = R[Q^\pi]$. 
With an estimation of $Q^{\pi}$, the expected reward $R^\pi$ in \eqref{equ:defRQ}
can be estimated by Monte Carlo sampling from $\mu_0$ and $\pi$.

\section{Main Method}
We now introduce our main framework for providing provably correct upper and lower bounds of the expected reward. 
For notation, we use $x = (s,a)$ to represent \zt{a} state-action pair.
\subsection{Motivation and Optimization Framework}
When the fitted value iteration in \eqref{equ:fittedQ} converges, it only provides one Q-function that yields a point estimation of the reward. 
To get an interval estimation, we expect the fitted value iteration to provide two Q-functions $\overline{Q}$ and $\underline{Q}$, 
such that all possible $Q$ consistent with the data points lie between $\overline{Q}$ and $\underline{Q}$, e.g. $\overline{Q}\succeq Q \succeq \underline{Q}$, \zt{where $f\succeq g$ means $f(x) \geq g(x),~\forall x$.} 

More concretely, consider a function set $\mathcal F$ that is expected to include $Q^\pi$, 
we construct an upper bound of $R^\pi$ by 
\begin{equation}
\label{eqn:upper_bound_opt_framework}
\overline{\Rpi_{\F}} 
= \sup_{Q\in \F} 
\left  \{
\zt{R_{\mu_{0,\pi}}}[Q], ~~
\text{s.t.}~~ Q(x_i) \leq \mathcal B^\pi Q(x_i),~\forall i\in [n]
\right \},
\end{equation}
where $R[Q]$ is defined in \eqref{equ:defRQ}. It is easy to see that $\overline{\Rpi_{\F}}  \geq R^\pi$ if $Q^\pi \in \mathcal F$ holds;
this is because $Q^\pi$ satisfies the constraints in the optimization, and hence 
$\overline{R^\pi_{\mathcal F}}\geq R[Q^\pi] = R^\pi$ as a result of the optimization.

It is worthy to note that we use the Bellman inequality constraint in \eqref{eqn:upper_bound_opt_framework}, 
which would not cause any looseness compared to the equality constraint. %
To see this, note that the exact reward $R^\pi$ can be framed into \citep{bertsekas1995dynamic}
\begin{equation*} 
\Rpi  
= \sup_{Q\in \F} 
\left  \{
\zt{R_{\mu_{0,\pi}}}[Q], ~~
\text{s.t.}~~ Q(x) \leq \mathcal B^\pi Q(x),~\forall 
x \in 
\mathcal S\times \mathcal A
\right \}\,, 
\end{equation*}
which implies that $\overline{R^\pi_{\mathcal F}}$ would converge to $\Rpi$ as data points increase.  %

We can construct the lower bound in a similar way: %
\begin{equation}
\label{eqn:lower_bound_opt_framework}
\underline{\Rpi_{\F}} 
= \inf_{Q\in \F} 
\left  \{
\zt{R_{\mu_{0,\pi}}}[Q], ~~
\text{s.t.}~~ Q(x_i) \geq \mathcal B^\pi Q(x_i),~\forall i\in [n]
\right \}. 
\end{equation}
Define $I_{\F} = [\underline{\Rpi_{\F}}, \overline{\Rpi_{\F}}]$ as the interval estimation for $\Rpi$, once the true $Q^\pi$ lies in the function space $\F$, $\Rpi$ lies in $I_{\F}$ provably.

\textbf{Benefits of Our Framework}\quad
Our optimization framework enjoys several advantages compared with the existing methods.
First of all, unlike the standard concentration bounds, our bounds do not rely on the i.i.d. assumption \zt{of transition pairs $\{s_i, a_i\}$.} 
\zt{In RL settings, the historical transition pairs are highly dependent to each other: on one hand, in sequential data stream, the current step of state is the next state in the previous step; on the other hand, the behavior policy that generates the trajectories is also evolving during the learning process.}
Secondly, under our optimization framework, more data would enable us to add more constraints in our searching space and therefore get tighter bounds accordingly. 
This property allows a trade-off between the time complexity and the tightness of the bounds, which we will further discuss in section \ref{sec:subsampling}.
Last but not least, the tightness of \zt{the} bounds depends on the capacity of the function space $\F$, \zt{which naturally yields the nice monotonicity property}. 
It is easy to see if $Q^\pi \in \F_1 \subset \F_2$, we will have $\Rpi \in I_{\F_1}\subset I_{\F_2}$.

\paragraph{Lipschitz Function Space} 
To implement this framework, 
it is important to choose a proper function set $\mathcal F$ and solve the optimization process efficiently.
Intuitively, \zt{$\F$}
should be rich enough to include the true value function, but not be too large to cause the bounds to be vacuous.  
In light of this, we propose to use the space of Lipschitz functions.  %
Let $\mathcal X$ be a metric space equipped with a distance $d\colon \mathcal X\times \mathcal X \to \RR$. 
We propose to take $\mathcal F$ to be a ball in the Lipschitz function space: 
\begin{align} \label{equ:defLip}
\mathcal F_\eta  := \{f \colon \|f\|_{d,\text{Lip}} \leq \eta\}, &&  \text{where} && 
\|f\|_{d,\text{Lip}} = 
\sup_{x,x'\in \mathcal X,~x\neq x'} \frac{|f(x)-f(x')|}{d(x,x')},
\end{align}
where $\|f\|_{d,\text{Lip}}$ is the Lipschitz norm of $f$ and 
$\eta$ is a radius parameter. 

Although the Lipschitz space yields an infinite dimensional optimization, our key technical contribution is to show that it is possible to calculate the upper and lower bounds with an efficient fixed point algorithm. 
\zt{We form it as} an optimization problem and solve it efficiently in a value iteration like fashion.   
In addition, in Section \ref{sec:tightness-lips}, we show that the size of the Lipschitz space enables us to establish a diminishing bound on the gap of the upper and lower bounds.

\begin{wrapfigure}{r}{.22\linewidth}
\vspace{-2.5\baselineskip}
    \centering
    \hspace{-.03\textwidth}
    \includegraphics[width = 1.1\linewidth]{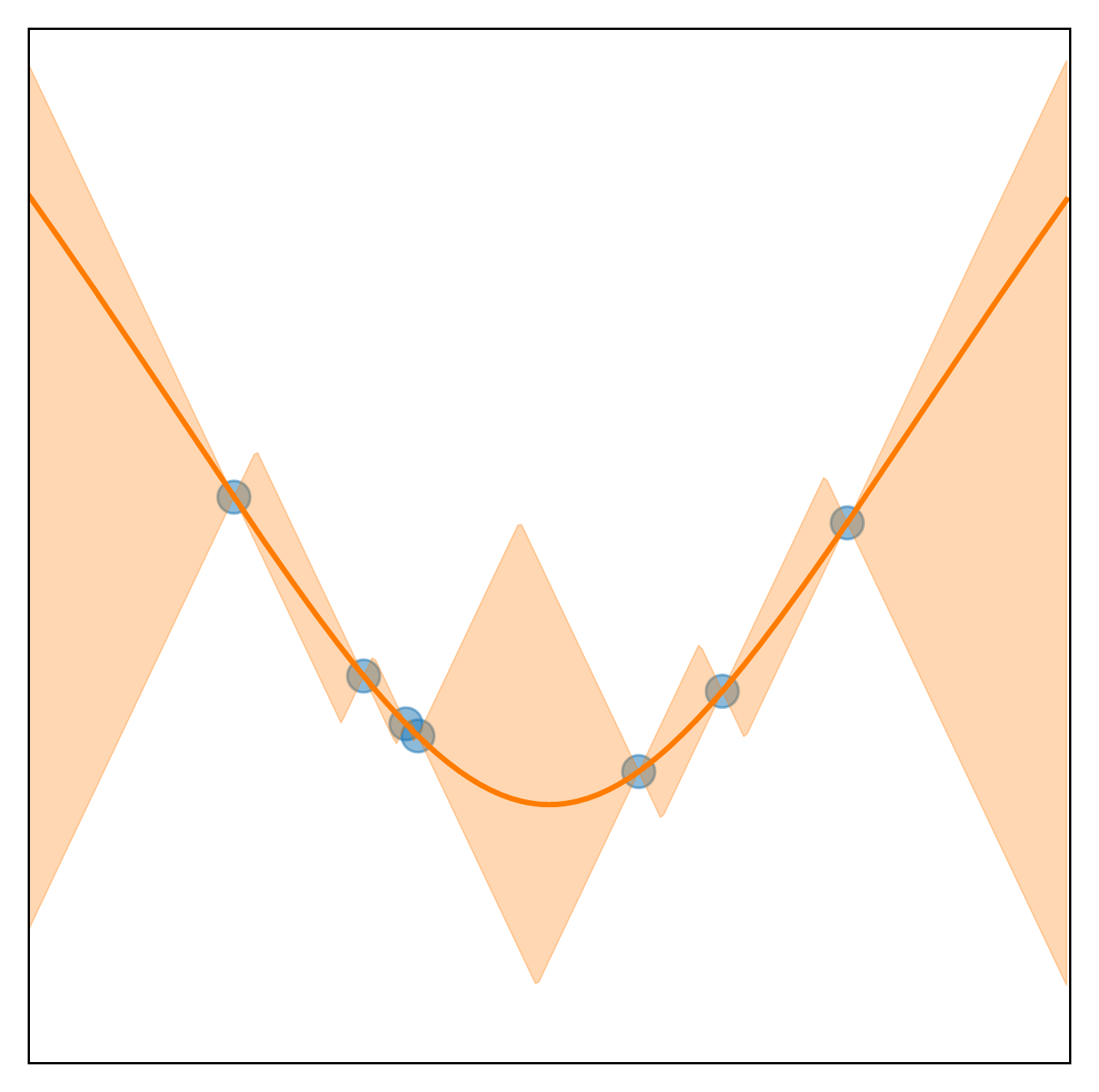}
    \vspace{-1.5em}
    \caption{
    \footnotesize{Illustration}
    }
    \vspace{-2.2em}
    \label{fig:example-lips-bound}
\end{wrapfigure}

\subsection{Optimization in Lipschitz Function Space}
\label{sec:alg_framework}

We propose a novel value iteration-like algorithm for solving the optimizations in \eqref{eqn:upper_bound_opt_framework}
 and \eqref{eqn:lower_bound_opt_framework}.
 We focus on the upper bound 
 \eqref{eqn:upper_bound_opt_framework}, as the case of the lower bound is similar and discussed in Appendix \ref{sec:lower_bound}.

Our algorithm enjoys the same spirit \zt{as} the fitted value iteration \eqref{equ:fittedQ}. We apply the Bellman operator on the current estimation of $Q$, and then use it as the constraint on the Bellman inequalities to obtain a new estimation.

\begin{algorithm*}[t] 
\caption{Lipschitz Value Iteration (for Upper Bound)}  
\label{alg:fitted_iteration}
\begin{algorithmic} 
\STATE {\bf Input}: Transition data $\D = \{ s_i, a_i, s_i', r_i\}_{1\leq i\leq n}$; 
discounted factor $\gamma$;
target policy $\pi$; Lipschitz constant $\eta$.
\STATE {\bf Initialize} 
 $\overline{q}_{0,i}$ with criterion like \eqref{equ:defq0}.  
\FOR{iteration $t$}%
\STATE {\bf Update:}
$
    \overline{q}_{t+1,i} = \B^\pi \overline{Q}_t(x_i),  \quad \text{where}~~\overline{Q}_t(x) = \min_{j\in[n]} \left \{ \overline{q}_{t,j} + \eta d(x, x_j) \right\} 
$
\ENDFOR
\STATE {\bf Return}: upper bound: $\overline{\Rpi_{\F}}  = \E_{s,a\sim \mu_{0,\pi}}\left [\min_{j\in[n]} \left \{ \overline{q}_{t,j} + \eta d(x, x_j)\right\}\right].$
\end{algorithmic} 
\end{algorithm*}

Specifically, starting from an initial $\overline{Q}_0$ with $\overline{q}_{0,i} := \overline{Q}_0(x_i)$, at the $t$-th iteration, we update by 
\begin{align}
\label{equ:algupdate}
{\overline{Q}}_{t}  = \arg\max_{Q\in \F}\left \{
\zt{R_{\mu_{0,\pi}}}[Q],~~
\text{s.t.}~~Q(x_i) \leq \overline{q}_{t,i}, \forall i \in [n]\right\},
&& \overline{q}_{t+1,i} =\mathcal B^{\pi} \overline{Q}_t(x_i), ~~~~~\forall i \in [n].
\end{align}
This update requires to solve a constrained optimization on the space $\F.$ 
For our choice of the Lipschitz function space, this yields a simple closed form solution. 
\begin{pro}
\label{pro:lips_update}
Suppose \zt{$\mu_{0,\pi}(x) > 0,~\forall x \in \Sset\times \Aset$}, 
consider the optimization in \eqref{equ:algupdate} with $\mathcal F = \mathcal F_\eta$. We have 
\begin{align} 
\label{eqn:lips_update}
    \overline{Q}_{t}(x) = \min_{j\in[n]}(\overline{q}_{t,j} + \eta d(x, x_j)),  
    &&\overline{q}_{t+1,i} = \mathcal B^\pi \overline{Q}_t(x_i), ~~~~\forall i \in [n].
\end{align}
\end{pro}
\vspace{-1em}
Intuitively, 
$\overline{Q}_t$ in \eqref{eqn:lips_update}  yields an 

\textit{upper envelope} of all the possible Lipschitz functions $Q$ that satisfy $Q(x_i) \leq  \overline{q}_{t,i},~\forall i$, and hence solves the optimization in \eqref{equ:algupdate}, as $R[Q]$ is monotonically increasing with $Q$.  
The updates for the lower bound can be derived similarly by calculating the \emph{lower envelopes}: 
\vspace{-0.6em}
\begin{align} 
\label{eqn:lips_update_lower}
    \underline{Q}_{t}(x) = \max_{i\in[n]}(\underline{q}_{t,i} - \eta d(x, x_i)),  
    &&\underline{q}_{t+1,i} = \mathcal B^\pi \underline{Q}_t(x_i), \forall i\in [n].
\end{align}

Note that these updates can be simplified to only keeping track of the Q-function values at the data points  $\{\overline{q}_{t,i}\}_{i=1}^n$, as summarized in Algorithm \ref{alg:fitted_iteration}.
Figure \ref{fig:example-lips-bound} illustrates how the upper and lower envelopes bound all the possible Lipschitz functions %
going through the same set of data points.

\subsection{Convergence Analysis}
Our algorithm enjoys two important and desirable properties: one is monotonic convergence, which indicates that you can stop any time, depending on your time budget, \zt{and} still get a valid lower/upper bound; the other is linear convergence, which implies that it only needs logarithmic time steps to converge. Thanks to these properties, our method is much faster compared to directly solving the convex optimization in \eqref{eqn:upper_bound_opt_framework} with (stochastic) sub-gradient ascent.

The monotonic convergence relies on the initial upper bound we pick, which is inspired by the monotonicity of the Bellman operator.
See Appendix \ref{sec:monotonic-proof} for more details.

\begin{thm}
\label{thm:monotonic}
Following the update in \eqref{eqn:lips_update} starting from 
\begin{align}\label{equ:defq0}
    \overline{q}_{0,i} = \frac{1}{1-\gamma}\left(r_i + \gamma \eta\E_{x_i'\sim \T^\pi(\cdot|x_i)}[d(x_i, x_i')]\right), 
\end{align}
we have
\begin{equation}
    \overline{Q}_{t}\succeq \overline{Q}_{t+1} \succeq \overline{Q}^\pi,~\forall t= 0, 1,2, \ldots,
\end{equation}
where $\overline{Q}^\pi = \arg\max_{Q\in \F}\{R[Q],~\text{s.t. } Q(x_i) \leq \B^\pi Q(x_i),~\forall i\in[n]\}$.
Therefore, 
$$
\zt{R_{\mu_{0,\pi}}}[\overline{Q}_{t}] \geq \zt{R_{\mu_{0,\pi}}}[\overline{Q}_{t+1}]  \geq  \overline{R^\pi_{\F}}\geq R^\pi, 
$$ 
and $\lim_{t\to \infty}\zt{R_{\mu_{0,\pi}}}[\overline{Q}_{t}] = \overline{R^\pi_{\F}}$.
\end{thm}

We establish a fast linear convergence rate for 
the updates in \eqref{eqn:lips_update}.
\zt{The convergence result here does not need the initialization of $\overline{q}_0$.} 
\begin{pro}
\label{pro:convergence_rate}
Following the updates in \eqref{eqn:lips_update} \zt{under arbitrary initialization}, 
with constant $C := \max_{i\in [n]}|\overline{q}_{1,i} - \overline{q}_{0,i}|$ we have 
$$
 \zt{R_{\mu_{0,\pi}}}[\overline{Q}_t] - \overline{\Rpi_\F} \leq C \frac{\gamma^t}{1-\gamma}.
$$
\end{pro}

\subsection{Tightness of Lipschitz-Based Bounds}
\label{sec:tightness-lips}

We provide a quantitative  estimation of  the gap $(\overline{\Rpi_{\F}} - \underline{\Rpi_{\F}})$ between the upper and lower bounds when using the Lipschitz function set. We show that it depends on a notion of the covering radius of the data points in the domain.

\begin{thm}
\label{thm:error_bound}
Let $\F = \F_{\eta}$ be the Lipschitz function class with Lipschitz constant $\eta$.
Suppose $\X = \Sset \times \Aset$ is a compact and bounded domain equipped with a distance $d\colon \X\times \X \to \RR$.
For a set of data points 
$X = \{x_i\}_{i=1}^n$, 
we have 
$$
\overline{\Rpi_{\F}} - \underline{\Rpi_{\F}} \leq \frac{2\eta}{1-\gamma} \varepsilon_X,
$$
where $\gamma$ is the discount factor and $\varepsilon_{X} = \sup_{x\in \X} \min_{i} d(x,x_i)$ is the covering radius.
\end{thm}
Typically, the covering radius in a bounded and compact domain asymptotically grows with $O(n^{-1/\tilde{d}})$, where $\tilde{d}$ is an intrinsic dimension of the domain \citep{cohen1985covering}, if the support of the data distribution (where $x_i$ is getting from) covers $d_\pi$, the stationary distribution of policy $\pi$.
This shows that the bound gets tighter as the number of samples get larger, but may decay slowly when the data domain has a very high intrinsic dimension. 
While it is possible to choose smaller space sets (such as RKHS) to obtain smaller gaps, it would sacrifice other properties such as capacity and simplicity.

\begin{algorithm*}[t] 
\caption{Lipschitz Value (Upper Bound) Iteration with Stochastic Update}  
\label{alg:lips}
\begin{algorithmic} 
\STATE {\bf Input}: Transition data $\D = \{s_i, a_i, s_i', r_i\}_{1\leq i\leq n}$;
discounted factor $\gamma$;
target policy $\pi$; 
distance function $d$.
Lipschitz constant $\eta$. Subsample size $n_B$.
\STATE {\bf Initialize}: $\overline{q}_{0,i} = \frac{1}{1-\gamma}(r_i + \gamma \eta \widehat{\E}_{x\sim \T^\pi(\cdot|x_i)}[d(x_i,x)]),~\forall i$, according to equation \eqref{equ:defq0}.
\FOR{iteration t}
\STATE Subsample a subset $S_t \subseteq \{1,2,...,n\}$ with $|S_t| = n_B$.
\STATE {\bf Update}: $\overline{q}_{t+1,i} = \min\{\overline{q}_{t,i}, \B^\pi \overline{Q}_{t}(x_i)\},~\forall i\in S_t$ and $\overline{q}_{t+1,i} = \overline{q}_{t,i},~\forall i\notin S_t$ where $\overline{Q}_{t}(x) = \min_{j\in S_t}\{\overline{q}_{t,j} + \eta d(x_j,x)\}$.
\ENDFOR
\STATE {\bf Return}: upper bound: $\overline{\Rpi_{\F}} = \widehat{\E}_{x\sim \mu_{0,\pi}}[\min_{i\in [n]}\{q_{T,i} + \eta d(x_i,x) \}]$.
\end{algorithmic} 
\end{algorithm*}

\section{Practical Considerations}

We discuss some practical concerns of Lipschitz value iteration in this section.
\subsection{Accelerating with Stochastic Subsampling}
\label{sec:subsampling}
If we draw $D$ samples to estimate $\B^\pi$ when updating with equation \eqref{eqn:lips_update}, we need $O(n^2 D)$ times of calculations for each round.
This $n^2$ comes from two sources, one is the updating of all $n$ terms of $q_{t,i}$ in each iteration, 
and the other is taking the maximum/minimum among all $n$ upper/lower curves when we calculate the upper/lower envelope function of $\overline{Q}_t(x)$. This is quadratic to the number of samples and therefore computationally expensive as the sample size grows large.

To tackle this problem, we propose a fast random subsampling technique. Instead of updating all $n$ $q_{t,i}$ in each iteration, we pick a batch of subsamples $\{\overline{q}_{t,i}\}_{i\in S}$ to update, where $S\in \{1,2,\ldots,n\}$ is a subset with fix size $|S| = n_B$.
The benefit of subsampling is that we can trade-off between time-complexity and tightness.

To be more precise, we can write down the new update scheme as follow:
\begin{align} 
\label{eqn:lips_update_subsample}
\begin{split} 
    &\overline{Q}_t(x) = \min_{j\in S_t}(\overline{q}_{t,j} + \eta d(x,x_j)),\\
    &\overline{q}_{t+1,i} = \min\{ \overline{q}_{t,i}, \B^\pi \overline{Q}_t(x_i),~\forall i\in S_t,\quad \overline{q}_{t+1,i} = \overline{q}_{t,i},~\forall i\notin S_t \},
\end{split}
\end{align}
where $S_t$ is the sub-sample set in the $t$ iteration.
A similar monotonic result can be shown in sequel.

\begin{pro}
\label{pro:monotonic_subsampling}
Consider the update in \eqref{eqn:lips_update_subsample} with initialization following \eqref{equ:defq0},
let $\overline{Q}_t$ be the upper envelope function of data points $\{x_i, \overline{q}_{t,i}\}_{i=1}^n$, we have a similar monotonic result as theorem~\ref{thm:monotonic}:
\begin{equation*}
    \overline{Q}_{t}\succeq \overline{Q}_{t+1} \succeq \overline{Q}^\pi,~\forall t= 0, 1,2, \ldots\,.
\end{equation*}
\end{pro}
We summarize the strategy in Algorithm \ref{alg:lips}.

The subsampling algorithm only needs $O(n_B^2 D)$ time complexity in each iteration.
Although this strategy does not converge to the exact optimal  solution $\overline{R^\pi_{\F}}$, it still gives valid (despite less tight) bounds.  
In the numerical experiments, we find that once the size of subsample set is sufficiently large, we get almost the same tightness as the exact optimal bounds.

\subsection{On Estimating Lipschitz Norm}
The only model assumption we need to specify is the function set $\F_{\eta}$; we want $\eta \geq  \|Q^{\pi}\|_{d,\mathrm{Lip}}$, but not to be too large, since the estimated interval gets loose as $\eta$ increases. 
However, compared to other value-based function approximation methods that also need assumptions on model specification, our non-parametric Lipschitz assumption is obviously very mild.

In order to set hyperparameter $\eta$, we would like to estimate the upper bound of $\|Q^\pi\|_{d,\mathrm{Lip}}$, which is typically non-identifiable purely from the data.
This is because, given a sufficiently large $\eta$, we can always find a function $Q$ that satisfies all the Bellman constraints with $\|Q\|_{d,\mathrm{Lip}} \geq \eta$ by twisting with a small function; see appendix \ref{sec:app-lips} for more details.
However, if we know or can estimate the Lipschitz norm of the reward and transition functions, then it is possible to derive a theoretical upper bound of $\|Q^\pi\|_{d,\mathrm{Lip}}$ with only mild assumptions.

\begin{figure*}[t]
    \centering
    \begin{tabular}{cccc}
    \raisebox{1.2em}{\rotatebox{90}{\scriptsize Relative log mean error}} 
    \hspace{-0.5em}
    \includegraphics[width = 0.22\textwidth]{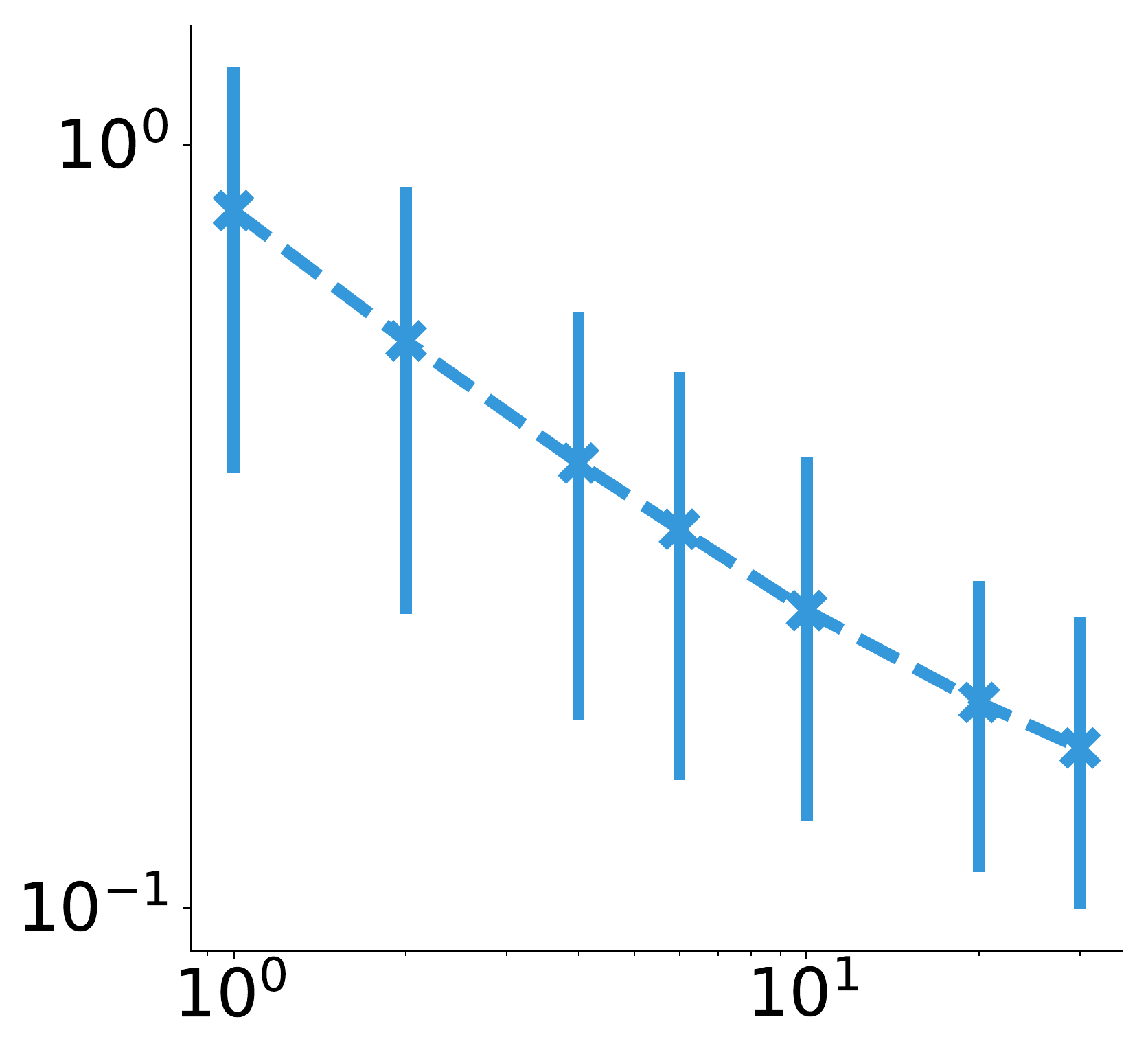}&
    \hspace{-0.5em}
    \raisebox{2.5em}{\rotatebox{90}{\scriptsize Relative Reward}} 
    \includegraphics[width = 0.22\textwidth]{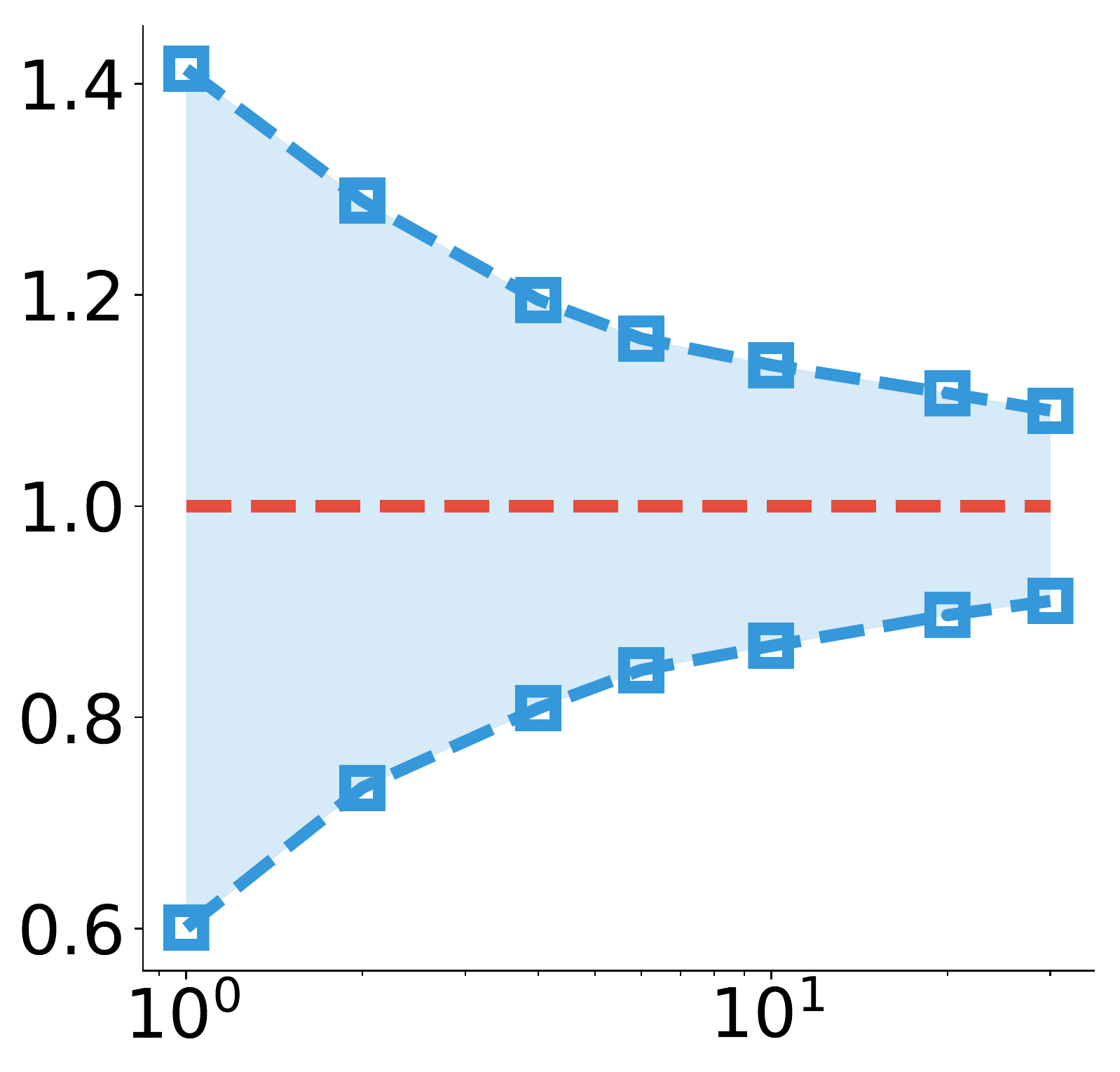}&
    \hspace{-0.5em}
    \includegraphics[width = 0.22\textwidth]{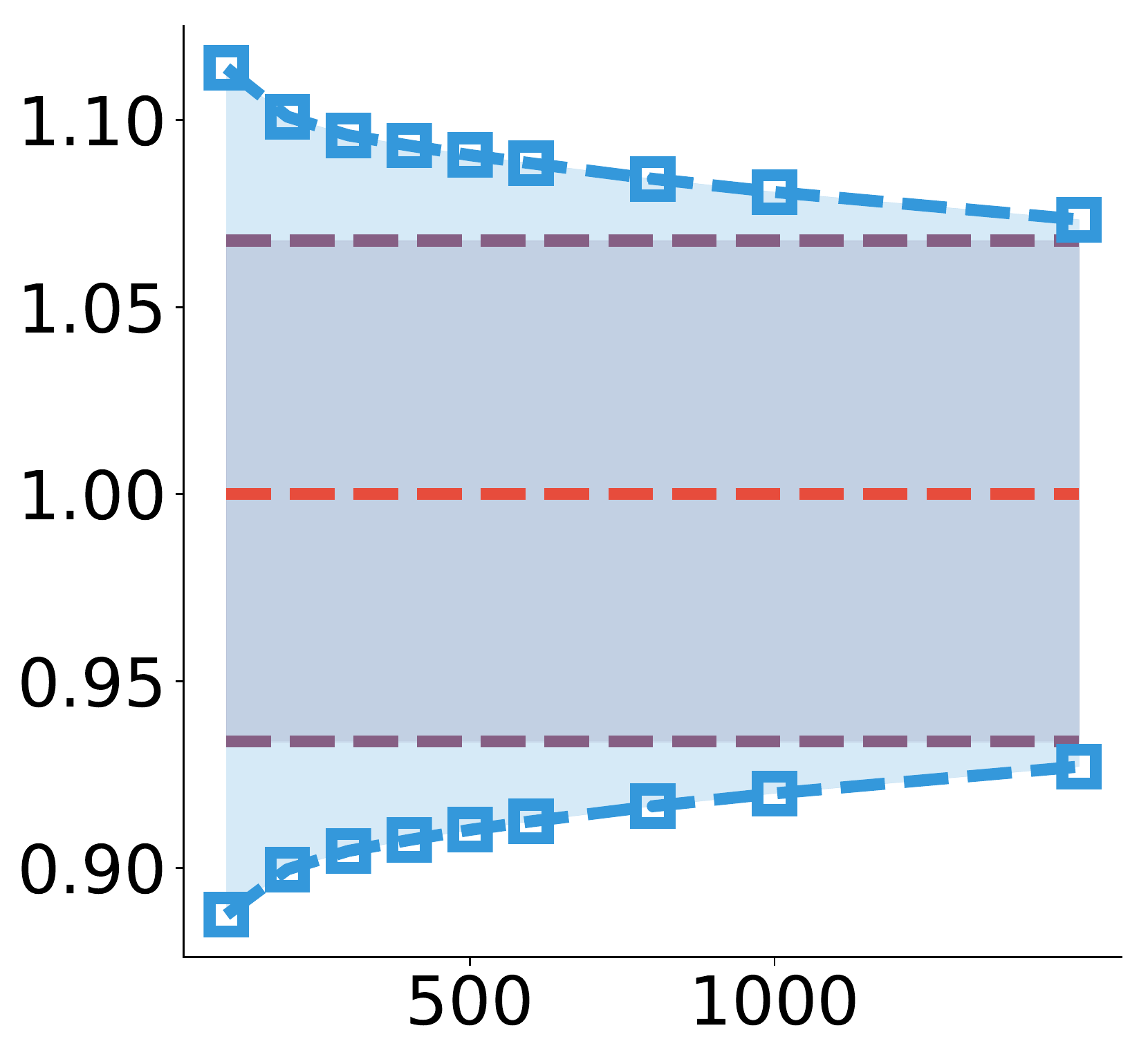}&
    \hspace{-0.5em}
    \includegraphics[width = 0.2\textwidth]{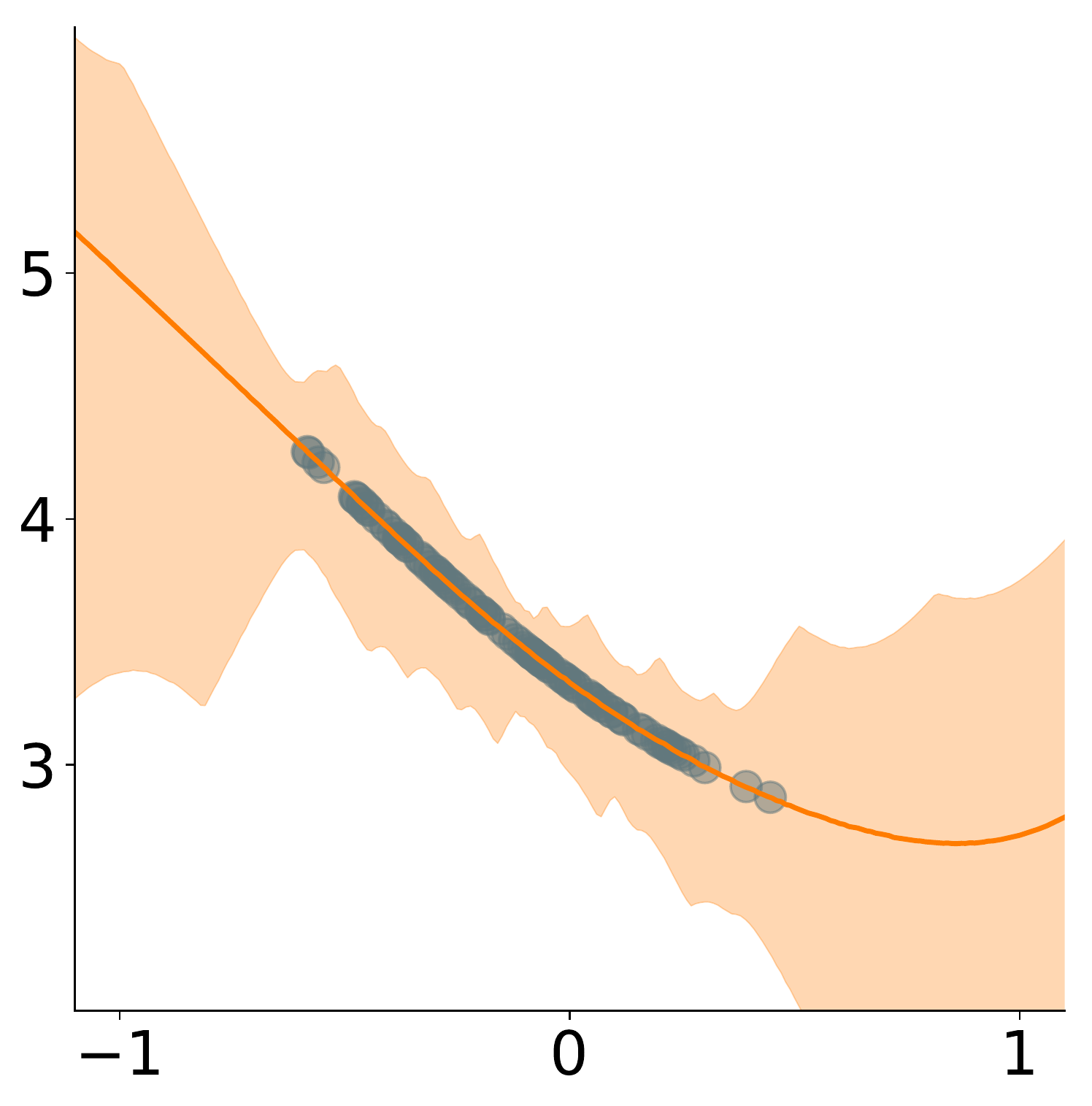}\\
    \small{number of trajectories $n_t$} & \small{number of trajectories $n_t$}  & \small{Subsample Size $n_B$} & \small{State $s$}\\
    (a) & (b) & (c)& (d)\\
    \end{tabular}
    \vspace{-0.8em}
    \caption{\small{Results for synthesis environment with a \textit{known} value function.
    The default settings: number of trajectory $n_t = 30,$ Horizon length $H = 100$, discounted factor $\gamma = 0.95$, Lipschitz constant $\eta = 2.0$ and subsample size $n_B = 500$.
    (a) y-axis: log of relative mean error $\log((\overline{R^\pi_\F} - \underline{R^\pi_\F})/\Rpi)$; (b)(c) y-axis: relative reward; (d) landscape for value function $V^\pi$ with $[\overline{V^\pi}, \underline{V^\pi}]$, here state is bounded in interval $[-1.2, 1.2]$, interval is estimated using 100 samples. Blue curves: subsamples bounds; purple lines: the whole samples bounds.
    }}
    \vspace{-1.3em}
    \label{fig:RL_toy}
\end{figure*}

\begin{pro}
\label{thm:Q_lips}
Let $\langle \Sset\times\Aset, d_x\rangle$ be a metric space for state action pair $x$ and $\langle \Sset, d_s\rangle$ be a metric space for state $s$.
Suppose $d_x$ is separable so that $d_x(x_1, x_2) = d_s(s_1, s_2)$ if $a_1 = a_2$.
If the reward function $r$ and the transition $\T$ are both Lipschitz in the sense that
\begin{align*}
    r(x_1) - r(x_2) \leq \|r\|_{\mathrm{Lip}} d_x(x_1,x_2),\quad
    d_s(\T(x_1), \T(x_2)) \leq \|\T\|_{\mathrm{Lip}} d_x(x_1, x_2),~\forall x_1, x_2. 
\end{align*}
We can prove that if $\gamma \|\T\|_{\mathrm{Lip}} < 1$, we have
\begin{equation}\label{eqn:lips_propagation}
    \|Q^\pi\|_{\mathrm{Lip}}\leq  \|r\|_{\mathrm{Lip}}/(1-\gamma \|\T\|_{\mathrm{Lip}}),
\end{equation}
when $\pi$ is a constant policy. 
Furthermore, for optimal policy $\pi^*$ with value function $Q^*$, we have:
\begin{equation}\label{eqn:lips_propagation_optimal}
    \|Q^*\|_{\mathrm{Lip}}\leq  \|r\|_{\mathrm{Lip}}/(1-\gamma \|\T\|_{\mathrm{Lip}}),
\end{equation}
\end{pro}

Theorem \ref{thm:Q_lips}
suggests that if our target policy is close to the optimal,
we can set $\eta = \frac{\|r\|_{\mathrm{Lip}}}{1-\gamma \|\T\|_{\mathrm{Lip}}}$ if we can estimate $\|r\|_{\mathrm{Lip}}$ and $\|\T\|_{\mathrm{Lip}}$. 
This provides a way for estimating the upper bound of the Lipschitz norm of $Q^\pi$ by leveraging the Lipschitz norm for the reward and transition functions.
In practice, we can estimate $\|r\|_{\mathrm{Lip}}$ and $\|\T\|_{\mathrm{Lip}}$ using historical data:
\begin{align}\label{eqn:estimate_lips1}
    \widehat{\|r\|}_{\mathrm{Lip}} = \max_{i\neq j} \frac{(r_i - r_j)}{d(x_i, x_j)}, && \widehat{\|\T\|}_{\mathrm{Lip}} = \max_{i\neq j} \frac{d_s(s_i', s_j')}{d_x(x_i, x_j)}. 
\end{align}

\paragraph{Diagnosing Model Misspecification from Data}
Since the empirical maximum tends to underestimate the true maximization,  %
simply using Proposition \ref{thm:Q_lips} may still underestimate the true Lipschitz norm.
Luckily, we can diagnose if $\eta$ is too small to be consistent with data by only adding a few lines of diagnosis codes.

From Theorem \ref{thm:monotonic} we know that for all $Q \in \F_{\eta}$, which are consistent with the finite sample Bellman equations, we have $\overline{Q}_t \succeq Q \succeq \underline{Q}_t$.
Thus, if at some time $t$ (or after convergence), we find that $\overline{Q}_t(x) < \underline{Q}_t(x)$ for some $x$, we can reject the following hypothesis:
$$
h: \exists Q\in \F_{\eta},~\mathrm{s.t.}~Q(x_i) = \B^\pi Q(x_i),~\forall i\in [n].
$$
In this way, we can see that $Q^\pi \notin \F_\eta$. Then, we can increase $\eta$ by a constant factor $\kappa > 1$ and rerun our upper/lower bound algorithm. 
Note that we do not need to compare an infinite number of $x$ to check $\overline{Q}_t(x) < \underline{Q}_t(x)$, as $\overline{Q}_t(x) = \min_{j}\{\overline{q}_{t,j} + \eta d(x,x_j)\}$ and $\underline{Q}_t(x) = \max_{j}\{\underline{q}_{t,j} - \eta d(x,x_j)\}$. And hence, it is sufficient to check if there exists an index $i$ such that $\overline{q}_{t,i} < \underline{q}_{t,i}$.

\begin{figure*}[t]
    \centering
    \begin{tabular}{cccc}
    \hspace{-0.6em}
    \raisebox{3.5em}{\rotatebox{90}{\scriptsize Relative Reward}} 
    \hspace{-0.5em}
    \includegraphics[width = 0.23\textwidth]{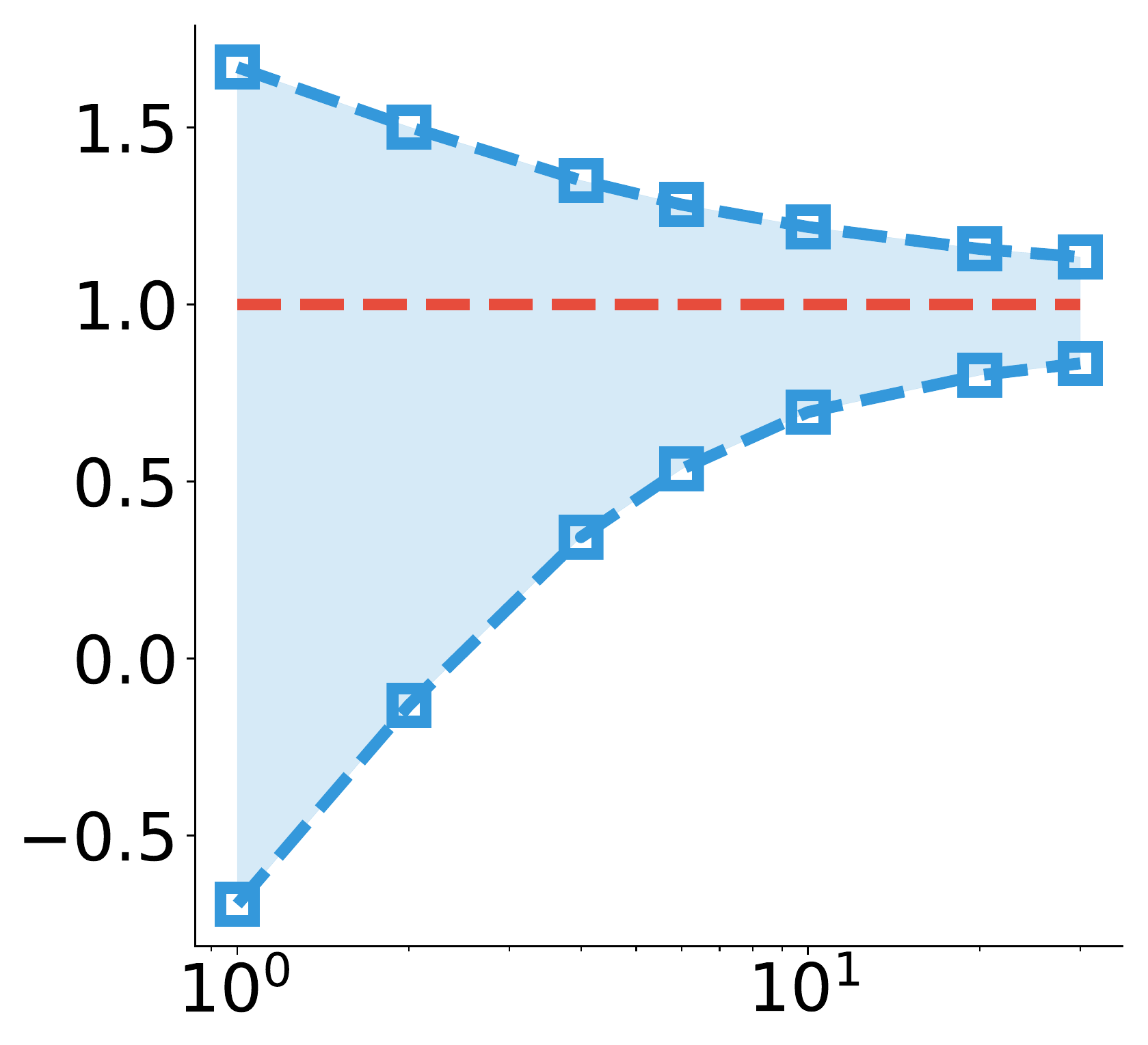}&
    \hspace{-0.6em}
    \includegraphics[width = 0.23\textwidth]{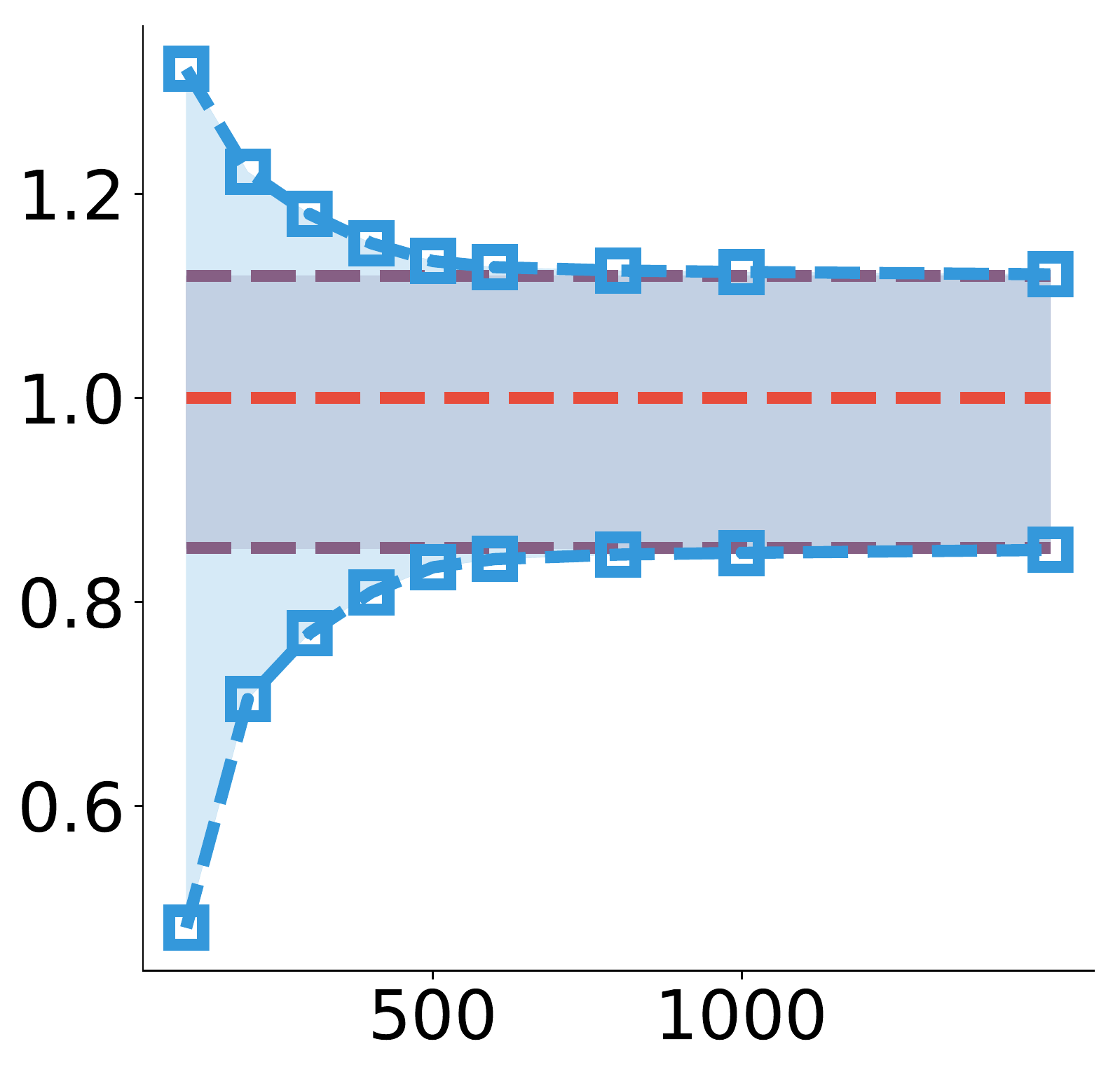}&
    \raisebox{3.5em}{\rotatebox{90}{\scriptsize Relative Reward}} 
    \hspace{-0.5em}
    \includegraphics[width = 0.23\textwidth]{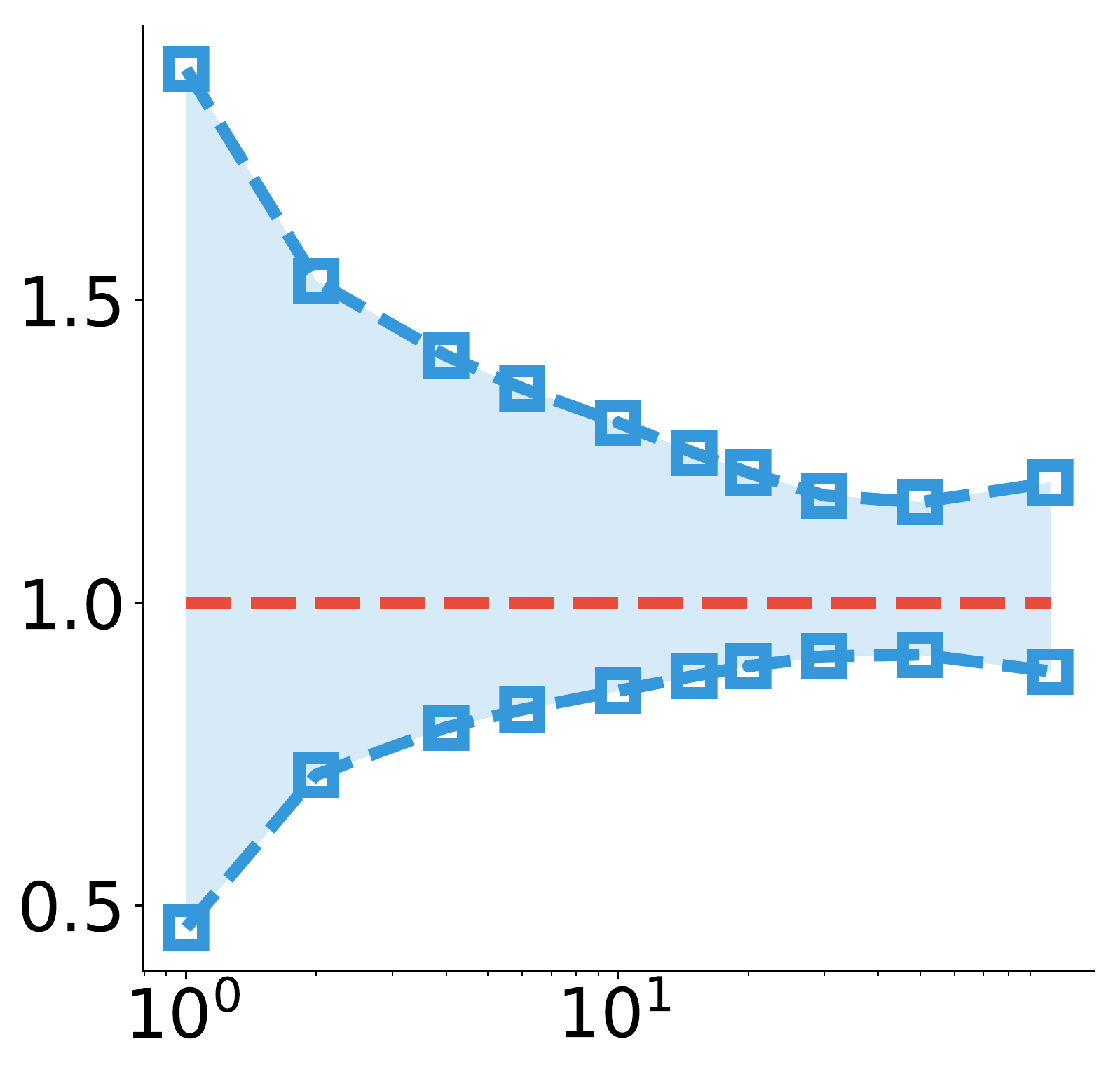}&
    \hspace{-0.6em}
    \includegraphics[width = 0.23\textwidth]{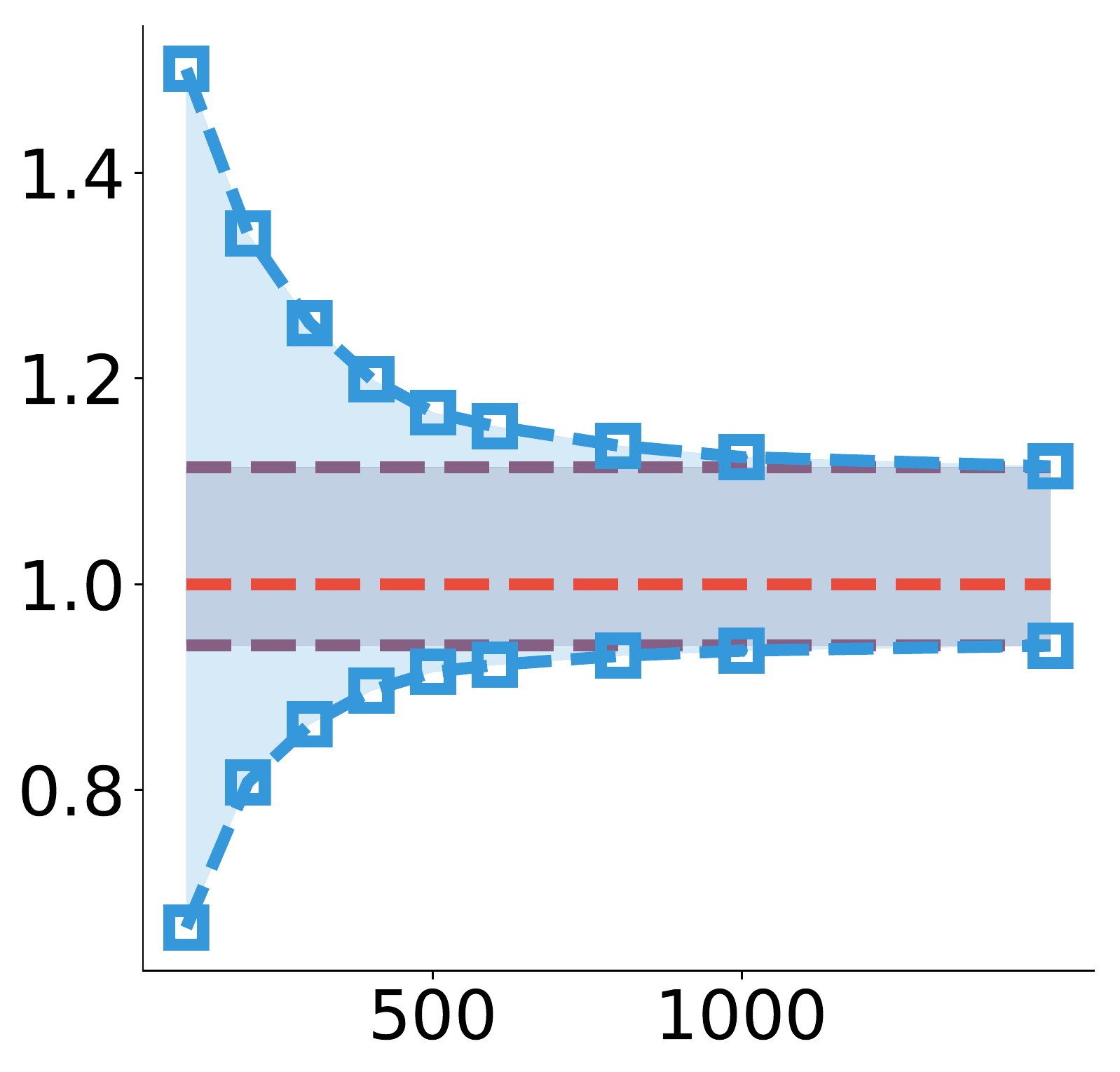}\\
    \small{number of trajectories $n_t$} & \small{Subsample Size $n_B$} & \small{number of trajectories $n_t$} & \small{Subsample Size $n_B$}\\
    (a) & (b) & (c) & (d)
    \end{tabular}
    \vspace{-1.3em}
    \caption{\small{Results for pendulum and HIV simulator.
    The default settings for pendulum (Figure a,b): $n_t = 30, H = 100, \gamma = 0.95, \eta = 10.0, n_B = 500$.
    The default settings for HIV (Figure c,d): $n_t = 50, H = 30, \gamma = 0.75, \eta = 40.0, n_B = 500$.
    We follow the similar experiments from Figure \ref{fig:RL_toy} (b) and (c).}
    }
    \vspace{-1.5em}
    \label{fig:RL_pendulum}
\end{figure*}

\section{Experiments}
\label{sec:experiments}
We test our algorithms in different environments.
We follow Algorithm $\ref{alg:lips}$ with sub-sampling technique.
In each environment, we evaluate the tightness of our bound 
by changing 
1) the number of samples $n$ and 
2) subsampling size $n_B$.
We also make a comparison with the exact bound with full sample size (i.e. $n=n_B$).
We start hyperparameter $\eta$ estimated by Proposition \ref{thm:Q_lips} with empirical maximal, and use diagnose algorithm in the last section to gradually increase $\eta$ until it is consistent with data set where we set the increasing factor $\kappa = 1.1$.
We observed that diagnosis algorithm usually passes with the very initial $\eta$.
As a baseline, \cite{thomas2015high} needs huge amounts of samples to get comparable tight bounds as ours, we demonstrate a comparison experiment in Appendix \ref{sec:experiment_details}.

\noindent\textbf{Synthetic Environment with A Known Value Function}\quad
We consider a simple environment with one dimension state space $\Sset = \RR$ and one dimension action space $\Aset = \RR$ with a linear transition function.
Given a target policy $\pi$, we enforce our value function $Q^\pi(s,a)$ to be our predefined function. 
This is done by enforcing the reward function $r$ for this environment as the reverse Bellman error of $Q^\pi$,
$
r_i = Q^\pi(x_i) - \gamma \Ppi Q^\pi(x_i),
$
with $\Ppi$ be the transition operator in equation \eqref{eqn:transition_op}.
We use Euclidean distance metric and under this metric we can prove that the Lipschitz constant is $2$. 
See Appendix \ref{sec:experiment_details} for more details.

All the reported results are average over 300 runs using the following setting by default: number of trajectory $n_t = 30,$ Horizon length $H = 100$, discounted factor $\gamma = 0.95$, Lipschitz constant $\eta = 2.0$ and subsample size $n_B = 500$.
We run Lipschitz value iteration for $100$ iteration to ensure almost convergence.
From Figure \ref{fig:RL_toy}(a)(b) we can see that as the number of sample size gets larger, the bound gets tighter.
\ref{fig:RL_toy}(c) indicates that with a sufficiently large subsample size, e.g. $n_B = 500$, we can achieve bounds accurately enough compared to whole sample algorithm (purple lines).
We also demonstrate the landscape of evaluation for state value function $\overline{V^\pi}(s) = \E_{a\sim \pi(\cdot|s)}[\overline{Q}_t(s,a)]$ and $\underline{V^\pi}(s)$ under the final Lipschitz value iteration for 100 data samples. 
Compare with the true value function, we can see that we get a tighter bound on a neighborhood region of data points compared to unseen region.

\noindent\textbf{Pendulum Environment}\quad
We demonstrate our method on pendulum, which is a continuous control environment with state space of $\RR^3$ and action space on interval $[-2,2]$.  
In this environment, we aim to control the pendulum to make it stand up as long as possible (for the large discounted case), or as fast as possible (for small discounted case). See Appendix \ref{sec:experiment_details} for more experimental setups.

Figure \ref{fig:RL_pendulum}(a)(b) shows a similar result indicating our interval estimation is tight and subsampling achieves almost same tightness with full samples.

\noindent\textbf{HIV Simulator}\quad
The HIV simulator described in \citet{ernst2006clinical} is a continuous state environment with 6 parameters and a discrete action environment with total 4 actions.
In this environment, we seek to find an optimal drug schedule given patient's 6 key HIV indicators.
The HIV simulator has richer dynamics than the previous two environments.
We follow \citet{liu2018representation} to learn a target policy by fitted Q iteration and use the $\epsilon$-greedy policy of the Q-function as the behavior policy.

The default setting is similar to the previous two experiments  but we use a relatively small discounted factor $\gamma = 0.75$ to ensure that we can get a reasonable Lipschitz constant from equation \eqref{eqn:lips_propagation}.
Figure \ref{fig:RL_pendulum}(c)(d) demonstrate a similar result of the HIV environment. 

\section{Conclusion}
We develop a general optimization framework for off-policy interval estimation and propose a value iteration style algorithm to monotonically tighten the interval.
Our Lipschitz value iteration on the continuous settings MDP enjoys nice convergence properties similar to the tabular MDP value iteration, which is worth further investigating.
Future directions include leveraging our interval estimation to encourage policy exploration or offline safe policy improvement. 

\paragraph{Broader Impact}
Off-policy interval evaluation not only can advise end-user to deploy new policy, but can also serve as an intermediate step for latter policy optimization.
Our proposed methods also fill in the gap of theoretical understanding of Markov structure in Lipschitz regression.
We current work stands as a contribution to the fundamental ML methodology, and we do not foresee potential negative impacts. 

\paragraph{Funding Transparency Statement}
This work is supported in part by NSF CAREER \#1846421, SenSE \#2037267, and EAGER \#2041327.

\bibliographystyle{icml2020}
\bibliography{ref}

\begin{thebibliography}{33}
\providecommand{\natexlab}[1]{#1}
\providecommand{\url}[1]{\texttt{#1}}
\expandafter\ifx\csname urlstyle\endcsname\relax
  \providecommand{\doi}[1]{doi: #1}\else
  \providecommand{\doi}{doi: \begingroup \urlstyle{rm}\Url}\fi

\bibitem[Bertsekas(2000)]{bertsekas1995dynamic}
Bertsekas, D.~P.
\newblock \emph{Dynamic Programming and Optimal Control}.
\newblock Athena Scientific, 2nd edition, 2000.
\newblock ISBN 1886529094.

\bibitem[Bottou et~al.(2013)Bottou, Peters, Qui{\~n}onero-Candela, Charles,
  Chickering, Portugaly, Ray, Simard, and Snelson]{bottou13counterfactual}
Bottou, L., Peters, J., Qui{\~n}onero-Candela, J., Charles, D.~X., Chickering,
  D.~M., Portugaly, E., Ray, D., Simard, P., and Snelson, E.
\newblock Counterfactual reasoning and learning systems: The example of
  computational advertising.
\newblock \emph{Journal of Machine Learning Research}, 14:\penalty0 3207--3260,
  2013.

\bibitem[Cohen et~al.(1985)Cohen, Karpovsky, Mattson, and
  Schatz]{cohen1985covering}
Cohen, G., Karpovsky, M., Mattson, H., and Schatz, J.
\newblock Covering radius---survey and recent results.
\newblock \emph{IEEE Transactions on Information Theory}, 31\penalty0
  (3):\penalty0 328--343, 1985.

\bibitem[Dann et~al.(2018)Dann, Li, Wei, and Brunskill]{dann2018policy}
Dann, C., Li, L., Wei, W., and Brunskill, E.
\newblock Policy certificates: Towards accountable reinforcement learning.
\newblock \emph{arXiv preprint arXiv:1811.03056}, 2018.

\bibitem[Ernst et~al.(2006)Ernst, Stan, Goncalves, and
  Wehenkel]{ernst2006clinical}
Ernst, D., Stan, G.-B., Goncalves, J., and Wehenkel, L.
\newblock Clinical data based optimal sti strategies for hiv: a reinforcement
  learning approach.
\newblock In \emph{Proceedings of the 45th IEEE Conference on Decision and
  Control}, pp.\  667--672. IEEE, 2006.

\bibitem[Feng et~al.(2019)Feng, Li, and Liu]{feng2019kernel}
Feng, Y., Li, L., and Liu, Q.
\newblock A kernel loss for solving the bellman equation.
\newblock \emph{Neural Information Processing Systems (NeurIPS)}, 2019.

\bibitem[Feng et~al.(2020)Feng, Ren, Tang, and Liu]{fengaccountable2020}
Feng, Y., Ren, T., Tang, Z., and Liu, Q.
\newblock Accountable off-policy evaluation with kernel bellman statistics.
\newblock In \emph{International Conference on Machine Learning}, 2020.

\bibitem[Fonteneau et~al.(2013)Fonteneau, Murphy, Wehenkel, and
  Ernst]{fonteneau13batch}
Fonteneau, R., Murphy, S.~A., Wehenkel, L., and Ernst, D.
\newblock Batch mode reinforcement learning based on the synthesis of
  artificial trajectories.
\newblock \emph{Annals of Operations Research}, 208\penalty0 (1):\penalty0
  383--416, 2013.

\bibitem[Hanna et~al.(2017)Hanna, Stone, and Niekum]{hanna2017bootstrapping}
Hanna, J., Stone, P., and Niekum, S.
\newblock Bootstrapping with models: Confidence intervals for off-policy
  evaluation.
\newblock In \emph{Proceedings of the 16th International Conference on
  Autonomous Agents and Multiagent Systems (AAMAS)}, May 2017.

\bibitem[Jiang \& Li(2016)Jiang and Li]{jiang16doubly}
Jiang, N. and Li, L.
\newblock Doubly robust off-policy evaluation for reinforcement learning.
\newblock In \emph{Proceedings of the 23rd International Conference on Machine
  Learning (ICML)}, pp.\  652--661, 2016.

\bibitem[Jin et~al.(2018)Jin, Allen-Zhu, Bubeck, and Jordan]{jin2018q}
Jin, C., Allen-Zhu, Z., Bubeck, S., and Jordan, M.~I.
\newblock Is q-learning provably efficient?
\newblock In \emph{Advances in Neural Information Processing Systems}, pp.\
  4863--4873, 2018.

\bibitem[Kallus \& Uehara(2019)Kallus and Uehara]{kallus2019efficiently}
Kallus, N. and Uehara, M.
\newblock Efficiently breaking the curse of horizon: Double reinforcement
  learning in infinite-horizon processes.
\newblock \emph{arXiv preprint arXiv:1909.05850}, 2019.

\bibitem[Le et~al.(2019)Le, Voloshin, and Yue]{le2019batch}
Le, H., Voloshin, C., and Yue, Y.
\newblock Batch policy learning under constraints.
\newblock In \emph{International Conference on Machine Learning}, pp.\
  3703--3712, 2019.

\bibitem[Li et~al.(2011)Li, Chu, Langford, and Wang]{li11unbiased}
Li, L., Chu, W., Langford, J., and Wang, X.
\newblock Unbiased offline evaluation of contextual-bandit-based news article
  recommendation algorithms.
\newblock In \emph{Proceedings of the 4th International Conference on Web
  Search and Data Mining (WSDM)}, pp.\  297--306, 2011.

\bibitem[Liu(2001)]{liu01monte}
Liu, J.~S.
\newblock \emph{{Monte Carlo} Strategies in Scientific Computing}.
\newblock Springer Series in Statistics. Springer-Verlag, 2001.
\newblock ISBN 0387763694.

\bibitem[Liu et~al.(2018{\natexlab{a}})Liu, Li, Tang, and
  Zhou]{liu2018breaking}
Liu, Q., Li, L., Tang, Z., and Zhou, D.
\newblock Breaking the curse of horizon: Infinite-horizon off-policy
  estimation.
\newblock In \emph{Advances in Neural Information Processing Systems}, pp.\
  5361--5371, 2018{\natexlab{a}}.

\bibitem[Liu et~al.(2018{\natexlab{b}})Liu, Gottesman, Raghu, Komorowski,
  Faisal, Doshi-Velez, and Brunskill]{liu2018representation}
Liu, Y., Gottesman, O., Raghu, A., Komorowski, M., Faisal, A.~A., Doshi-Velez,
  F., and Brunskill, E.
\newblock Representation balancing mdps for off-policy policy evaluation.
\newblock In \emph{Advances in Neural Information Processing Systems}, pp.\
  2644--2653, 2018{\natexlab{b}}.

\bibitem[Mousavi et~al.(2020)Mousavi, Li, Liu, and Zhou]{mousavi2020blackbox}
Mousavi, A., Li, L., Liu, Q., and Zhou, D.
\newblock Black-box off-policy estimation for infinite-horizon reinforcement
  learning.
\newblock In \emph{International Conference on Learning Representations}, 2020.

\bibitem[Munos \& Szepesv{\'a}ri(2008)Munos and
  Szepesv{\'a}ri]{munos2008finite}
Munos, R. and Szepesv{\'a}ri, C.
\newblock Finite-time bounds for fitted value iteration.
\newblock \emph{Journal of Machine Learning Research}, 9\penalty0
  (May):\penalty0 815--857, 2008.

\bibitem[Murphy et~al.(2001)Murphy, van~der Laan, and Robins]{murphy01marginal}
Murphy, S.~A., van~der Laan, M., and Robins, J.~M.
\newblock Marginal mean models for dynamic regimes.
\newblock \emph{Journal of the American Statistical Association}, 96\penalty0
  (456):\penalty0 1410--1423, 2001.

\bibitem[Nachum et~al.(2019)Nachum, Chow, Dai, and Li]{nachum2019dualdice}
Nachum, O., Chow, Y., Dai, B., and Li, L.
\newblock Dualdice: Behavior-agnostic estimation of discounted stationary
  distribution corrections.
\newblock In \emph{Advances in Neural Information Processing Systems}, pp.\
  2315--2325, 2019.

\bibitem[Precup et~al.(2000)Precup, Sutton, and Singh]{precup00eligibility}
Precup, D., Sutton, R.~S., and Singh, S.~P.
\newblock Eligibility traces for off-policy policy evaluation.
\newblock In \emph{Proceedings of the 17th International Conference on Machine
  Learning (ICML)}, pp.\  759--766, 2000.

\bibitem[Song \& Sun(2019)Song and Sun]{song2019efficient}
Song, Z. and Sun, W.
\newblock Efficient model-free reinforcement learning in metric spaces.
\newblock \emph{arXiv preprint arXiv:1905.00475}, 2019.

\bibitem[Sutton \& Barto(1998)Sutton and Barto]{sutton98beinforcement}
Sutton, R.~S. and Barto, A.~G.
\newblock \emph{Reinforcement Learning: An Introduction}.
\newblock MIT Press, Cambridge, MA, March 1998.
\newblock ISBN 0-262-19398-1.

\bibitem[Tang et~al.(2020)Tang, Feng, Li, Zhou, and Liu]{tang2020doubly}
Tang, Z., Feng, Y., Li, L., Zhou, D., and Liu, Q.
\newblock Doubly robust bias reduction in infinite horizon off-policy
  estimation.
\newblock In \emph{International Conference on Learning Representations}, 2020.

\bibitem[Thomas \& Brunskill(2016)Thomas and Brunskill]{thomas2016data}
Thomas, P. and Brunskill, E.
\newblock Data-efficient off-policy policy evaluation for reinforcement
  learning.
\newblock In \emph{International Conference on Machine Learning}, pp.\
  2139--2148, 2016.

\bibitem[Thomas et~al.(2015{\natexlab{a}})Thomas, Theocharous, and
  Ghavamzadeh]{thomas2015bhigh}
Thomas, P., Theocharous, G., and Ghavamzadeh, M.
\newblock High confidence policy improvement.
\newblock In \emph{International Conference on Machine Learning}, pp.\
  2380--2388, 2015{\natexlab{a}}.

\bibitem[Thomas et~al.(2015{\natexlab{b}})Thomas, Theocharous, and
  Ghavamzadeh]{thomas2015high}
Thomas, P.~S., Theocharous, G., and Ghavamzadeh, M.
\newblock High-confidence off-policy evaluation.
\newblock In \emph{Twenty-Ninth AAAI Conference on Artificial Intelligence},
  2015{\natexlab{b}}.

\bibitem[Thomas et~al.(2017)Thomas, Theocharous, Ghavamzadeh, Durugkar, and
  Brunskill]{thomas17predictive}
Thomas, P.~S., Theocharous, G., Ghavamzadeh, M., Durugkar, I., and Brunskill,
  E.
\newblock Predictive off-policy policy evaluation for nonstationary decision
  problems, with applications to digital marketing.
\newblock In \emph{Proceedings of the 31st AAAI Conference on Artificial
  Intelligence (AAAI)}, pp.\  4740--4745, 2017.

\bibitem[White \& White(2010)White and White]{white2010interval}
White, M. and White, A.
\newblock Interval estimation for reinforcement-learning algorithms in
  continuous-state domains.
\newblock In \emph{Advances in Neural Information Processing Systems}, pp.\
  2433--2441, 2010.

\bibitem[Xie et~al.(2019)Xie, Ma, and Wang]{xie2019optimal}
Xie, T., Ma, Y., and Wang, Y.-X.
\newblock Optimal off-policy evaluation for reinforcement learning with
  marginalized importance sampling.
\newblock \emph{Neural Information Processing Systems (NeurIPS)}, 2019.

\bibitem[Yang et~al.(2019)Yang, Ni, and Wang]{yang2019learning}
Yang, L.~F., Ni, C., and Wang, M.
\newblock Learning to control in metric space with optimal regret.
\newblock \emph{arXiv preprint arXiv:1905.01576}, 2019.

\bibitem[Zhang et~al.(2020)Zhang, Dai, Li, and Schuurmans]{zhang2020gendice}
Zhang, R., Dai, B., Li, L., and Schuurmans, D.
\newblock Gen\uppercase{DICE}: Generalized offline estimation of stationary
  values.
\newblock In \emph{International Conference on Learning Representations}, 2020.

\end{thebibliography}

\newpage\clearpage
\onecolumn
\appendix

\begin{center}
\Large
\textbf{Appendix}
\end{center}

\section{Lower Bound Results}
\label{sec:lower_bound}
We list all the results for the lower bound here.

\subsection{The Lower Bound Value Iteration}
Similar to upper bound, the general algorithm for lower bound iteratively finds the lower envelope of the previous points estimation $\underline{q}_{t,i}$.

\begin{align}
\label{equ:algupdate_lower}
\begin{split} 
&{\underline{Q}}_{t}  = \arg\min_{Q\in \F}\left \{
\zt{R_{\mu_{0,\pi}}}[Q],~~~~
\text{s.t.}~~~~Q(x_i) \geq \underline{q}_{t,i},
~~~~ \forall i \in [n]\right\} \\
& \underline{q}_{t+1,i} =\mathcal B^{\pi} \underline{Q}_t(x_i), ~~~~~\forall i \in [n].
\end{split}
\end{align}

For Lipschitz functions, this yields a simple closed form solution. 
\begin{pro}
\label{pro:lips_update_lower}
Suppose $\mu_{0,\pi}$ is a full-support distribution over $\Sset\times \Aset$. 
Consider the optimization in \eqref{equ:algupdate_lower} with $\mathcal F = \mathcal F_\eta$ in \eqref{equ:defLip}. We have 
\begin{align} 
\label{eqn:lips_update_lower_appendix}
\begin{split} 
    & \underline{Q}_{t}(x) = \max_{j\in[n]}(\underline{q}_{t,j} - \eta d(x, x_j)),  \\
    &\underline{q}_{t+1,i} = \mathcal B^\pi \underline{Q}_t(x_i), ~~~~\forall i \in [n]\qq{.}
    \end{split}
\end{align}
\end{pro}

\subsection{Convergence Results}

Similar to Theorem \ref{thm:monotonic}, we have a similar monotonic result for lower bound case.
\begin{thm}
\label{thm:monotonic_lower}
Following the update in \eqref{eqn:lips_update_lower} starting from 
\begin{align}\label{equ:defq0_lower}
    \underline{q}_{0,i} = \frac{1}{1-\gamma}\left(r_i - \gamma \eta\E_{x_i'\sim \T^\pi(\cdot|x_i)}[d(x_i, x_i')]\right), 
\end{align}
we have
\begin{equation}
    \underline{Q}_{t}\preceq \underline{Q}_{t+1} \preceq \underline{Q}^\pi,~\forall t= 0, 1,2, \ldots,
\end{equation}
where $\underline{Q}^\pi = \arg\min_{Q\in \F}\{R[Q],~\text{s.t. } Q(x_i) \geq \B^\pi Q(x_i),~\forall i\in[n]\}$.
Therefore, 
$$
\zt{R_{\mu_{0,\pi}}}[\underline{Q}_{t}] \leq \zt{R_{\mu_{0,\pi}}}[\underline{Q}_{t+1}]  \leq  \underline{R^\pi_{\F}}\leq R^\pi, 
$$ 
and $\lim_{t\to \infty}\zt{R_{\mu_{0,\pi}}}[\underline{Q}_{t}] = \underline{R^\pi_{\F}}$.
\end{thm}

Similar to the linear convergence property of the upper bound case, we can establish a fast linear convergence rate for 
the updates in \eqref{eqn:lips_update_lower}. 
\begin{pro}
\label{pro:convergence_rate_lower}
Following the updates in \eqref{eqn:lips_update_lower}, we have  
$$
 \underline{\Rpi_\F} - \zt{R_{\mu_{0,\pi}}}[\underline{Q}_t] \leq C \frac{\gamma^t}{1-\gamma},
$$
with constant $C := \max_{i\in [n]}|\underline{q}_{1,i} - \underline{q}_{0,i}|$. 
\end{pro}

\section{Proofs}
\label{sec:appendix-proof}
We focus on the proofs for upper bounds, 
all the lower bound proofs follows similarly. 

We establish the monotonic convergence of the 
 iterative update in section \ref{sec:alg_framework}. 
 We start with the result for general function spaces $\F$ and then apply it to the case of Lipschitz functions space, 
 where $\F$ can ensure that the optimization in \eqref{equ:algupdate} is solved by a properly defined upper envelope function.

\newcommand{\upenv}{\mathrm{\overline{ENV}}}
\newcommand{\loenv}{\mathrm{\underline{ENV}}}

\begin{mydef}\label{def:super}
Given a function space $\F$ on domain $\Omega$ 
and a set of data points $(x_i,f_i)_{i=1}^n\subseteq \Omega \times \RR$, %
we define the upper envelope function $g:\Omega\to \RR$  
of data points $(x_i, f_i)_{i=1}^n$ on $\mathcal F$ to be 
$$
g(x) = \upenv_{\F}(\{x_i,f_i\})(x) := \sup_{f\in \F}\{f(x):~\text{s.t.}~~ f(x_i)\leq f_i,~~\forall i \in [n]\},~\forall x\in \Omega. $$
We say that $\mathcal F$ is upper-self-contained
if it is closed (under the infinity norm), 
and the upper envelope function $g:\Omega\to R$ is contained in $\F$ for 
any data points $(x_i, f_i)$  
that satisfies  $\inf_{f\in \F} \{f(x_i)\} \leq f_i,~\forall i\in [n]$. 

Similarly,
we define the lower envelope function $g:\Omega\to \RR$  
of data points $(x_i, f_i)_{i=1}^n$ on $\mathcal F$ to be 
$$
g(x) = \loenv_{\F}(\{x_i,f_i\})(x):= \inf_{f\in \F}\{f(x):~\text{s.t.}~~ f(x_i)\geq f_i,~~\forall i \in [n]\},~\forall x\in \Omega. $$
We say that $\mathcal F$ is lower-self-contained
if it is closed (under the infinity norm), 
and the lower envelope function $g:\Omega\to R$ is contained in $\F$ for 
any data points $(x_i, f_i)$  
that satisfies  $\sup_{f\in \F} \{f(x_i)\} \geq f_i,~\forall i\in [n]$. 
\end{mydef}

Similar to contractive operator proof in value iteration, we also define the contractive property of upper/lower envelope operator $\upenv_{\F}$ and $\loenv_{\F}$.

\begin{mydef}
We say $\upenv_{\F}$ and $\loenv_{\F}$ are contractive if for two different sets of points data $\{x_i, p_i\}_{i=1}^{n}$ and $\{x_i,q_i\}_{i=1}^n$, we have, 
$$
\|\upenv_{\F}(\{x_i, p_i\}_{i=1}^{n}) - \upenv_{\F}(\{x_i, q_i\}_{i=1}^{n})\|_{\infty}\leq \max_{i\in [n]} |p_i - p_i|.\qq{,}
$$
and
$$
\|\loenv_{\F}(\{x_i, p_i\}_{i=1}^{n}) - \loenv_{\F}(\{x_i, q_i\}_{i=1}^{n})\|_{\infty}\leq \max_{i\in [n]} |p_i - p_i|.
$$
\end{mydef}

The following lemmas provide key
properties for of this special function class.
If $\F$ is upper-self-contained, then the optimization in \eqref{equ:algupdate} is %
solved by the upper envelop function defined above. 
And similarly, if $\F$ is lower-self-contained, then the optimization in \eqref{equ:algupdate_lower} is %
solved by the lower envelop function defined above. 
\begin{lem}
\label{lem:upper-contained}
If $\F$ is upper-self-contained/lower-self-contained, then 
$\overline{Q}_t$(resp. $\underline{Q}_t$)
in \eqref{equ:algupdate}(resp. \eqref{equ:algupdate_lower})
is equal to upper(resp. lower) envelope function of data points $(x_i, \overline{q}_{t,i}$(resp. $\underline{q}_{t,i})_{i=1}^n$) almost everywhere.
\end{lem}
\begin{proof}
    The upper envelope and lower envelope is inside the function space, and is maximized(resp. minimized) 
    for all data points. 
    Therefore they are the solutions to equation \eqref{equ:algupdate} and \eqref{equ:algupdate_lower} almost everywhere, respectively.
\end{proof}

In addition, the upper and lower envelope functions is monotonic w.r.t. the data labels it goes through. 
\begin{lem}
\label{lem:point2function}
In an upper-self-contained function space $\F$, suppose we have two sets of data points $(x_i, f_i)_{i=1}^n$ and $(x_i, g_i)_{i=1}^n$, and $\overline{f}$ and $\overline{g}$ are their upper envelope functions respectively, $\underline{f}$ and $\underline{g}$ are their lower envelopes functions respectively,
if $f_i \geq g_i,~\forall i\in [n]$, then we have 
$
\overline{f} \succeq \overline{g}, \quad \underline{f} \succeq \underline{g}.
$
\end{lem}
    \begin{proof} 
    This is directly from the definition,
    \begin{align*}
        \overline{f}(x) =& \max_{f\in \F}\{f(x):~\text{s.t.} f(x_i) \leq f_i\} \\
        \geq& \max_{f\in \F}\{f(x):~\text{s.t.} f(x_i) \leq g_i\} \\
        =& \overline{g}(x),
    \end{align*}
    where the first inequality holds because the feasible region of constraints $f(x_i) \leq f_i$ is more general than $f(x_i) \leq g_i$ when $f_i\geq g_i$.
    
    the proof for the lower envelope works similarly.
    \end{proof}

\begin{lem}\label{lem:inequality2equality}
For a bounded upper-self-contained function class $\F$, if the upper envelope operator is contractive, then the maximum solution $\overline{Q}^\pi$ for the optimization framework equation \eqref{eqn:upper_bound_opt_framework} is the unique solution of the following upper-envelope Bellman equation:
\begin{align}\label{eqn:upper-envelope-Bellman}
\begin{split}
    Q_i =& \B^\pi Q(x_i),~\forall i\in[n],\\
    Q =& \upenv(\{x_i, Q_i\}_{i=1}^n).
\end{split}
\end{align}
\end{lem}
    \begin{proof}
    \textbf{Existence}
    
    Suppose $P$ is a optimum solution for \eqref{eqn:upper_bound_opt_framework}, if $P$ satisfies equation \eqref{eqn:upper-envelope-Bellman} then we are done.
    Otherwise $P$ satisfies $P\in \F$ and $P(x_i)\leq \B^\pi P(x_i),~\forall i\in[n]$.
    Consider $q^\dag_{i} = \B^\pi P(x_i)$, its corresponding upper-envelope function $Q^\dag$ satisfies:
    $$
    Q^\dag(x) = \max_{Q\in \F}\{Q(x),~\text{s.t.}~Q(x) \leq \B^\pi P(x_i)\} \geq P(x).
    $$
    Thus $Q^\dag \succeq P$.
    By Bellman inequality we have:
    $$
    \B^\pi Q^\dag(x_i) \geq \B^\pi P(x_i) \geq Q^\dag(x_i),
    $$
    We have $Q^\dag$ is in $\F$ and also satisfies Bellman inequality.
    By repeating this process until it converge to $Q^{\infty}$, we will eventually get $Q^\infty(x_i)$ satisfies equation \eqref{eqn:upper-envelope-Bellman} and
    $Q^\infty \succeq P$ which means $Q^\infty$ is also at least a optimal solution to optimization framework \eqref{eqn:upper_bound_opt_framework}.
    
    \textbf{Uniqueness}
    
    Consider there are two functions $Q_1$ and $Q_2$ satisfy upper-envelope Bellman equation in \eqref{eqn:upper-envelope-Bellman}.
    Consider $q_i^{k} = \B^\pi Q^{k}(x_i),~\forall i\in [n], k\in \{1,2\}$, and let $q^k$ to denote the vector of $[q_1^k, q_2^k,...,q_n^k]^\top$, we have the infinity norm of $q^1 - q^2$ to be:
    \begin{align*}
        q_i^1 - q_i^2 =& \gamma \P^\pi (Q^1(x_i) - Q^2(x_i)) \\
        =& \gamma \E_{x'\sim \T^\pi(\cdot|x_i)}[Q^1(x') - Q^2(x')] \\
    =& \gamma \E_{x'\sim \T^\pi(\cdot|x_i)}[\max_{P\in \F}\{P(x),~\text{s.t.}~ P(x_j)\leq q_j^1,\forall j\in [n]\} - \max_{P\in \F}\{P(x),~\text{s.t.}~ P(x_j)\leq q_j^2,\forall j\in [n]\}] \\
    \leq& \gamma \|q^1 - q^2\|_{\infty},
    \end{align*}
    where the last inequality is from contractive property.
    This means $\|q^1 - q^2\|_{\infty} = 0$, and since $Q^1, Q^2$ are there corresponding upper-envelope, we have $Q^1 = Q^2$.
    \end{proof}

\paragraph{Proposition \ref{pro:lips_update} (and \ref{pro:lips_update_lower})}
    {\itshape
    Suppose $\mu_{0,\pi}$ is a full-support distribution over $\Sset\times \Aset$.
    Consider the optimization in \eqref{equ:algupdate} with $\mathcal F = \mathcal F_\eta$ in \eqref{equ:defLip}. We have 
    \begin{align} 
    \begin{split} 
        & \overline{Q}_{t}(x) = \min_{j\in[n]}(\overline{q}_{t,j} + \eta d(x, x_j)),  \\
        &\overline{q}_{t+1,i} = \mathcal B^\pi \overline{Q}_t(x_i), ~~~~\forall i \in [n]
        \end{split}
    \end{align}
    }
    \begin{proof}
    Consider $Q\in \F_{\eta}$, for upper bound case, we have:
    $$
    Q(x) \leq Q(x_i) + \eta d(x,x_i) \leq \overline{q}_{t,i} + \eta d(x,x_i),~\forall i\in [n].
    $$
    Therefore $Q(x) \leq \min_{i\in [n]}\{\overline{q}_{t,i} + \eta d(x,x_i)\}$.
    
    Consider the upper envelope function $\overline{Q}_t$ which achieves $\overline{Q}_t(x) = \min_{i\in [n]}\{\overline{q}_{t,i} + \eta d(x,x_i)\}$.
    
    By Lemma \ref{lem:upper-contained} we have:
    $$
    \overline{Q}_t = \arg\max_{Q\in \F_{\eta}} \{R[Q],~~\text{s.t.}~~Q(x_i)\leq \overline{q}_{t,i}\}.
    $$
    
    Similarly we can prove for the lower bound case in Proposition \ref{pro:lips_update_lower}.
    \end{proof}
    
\subsection{Monotonic Convergence}
\label{sec:monotonic-proof}

It is well known that the Bellman operator is a contractive map when $\gamma \in (0,1)$, with $Q^\pi$ as the unique fixed point. 
Therefore,  $(\mathcal B^\pi)^{t} Q$ converges to $Q^\pi$ as $t\to \infty$ for any $Q$.

Another property of special importance in our work is the 
monotonicity of Bellman operator. 
For two functions $Q_1$ and $Q_2$ on $\mathcal S \times \mathcal A$, 
we say that $Q_1 \succeq Q_2$ if $Q_1(s,a) \geq Q_2(s,a)$ for $\forall (s,a)\in \mathcal S \times \mathcal A$. Then we have 
\begin{equation*}
    Q_1 \succeq Q_2 ~~~~\Rightarrow ~~~~\B^\pi Q_1 \succeq \B^\pi Q_2. 
\end{equation*}
Thus if we can establish $Q\succeq \B^\pi Q$ (which known as the Bellman inequality \citep{bertsekas1995dynamic}), 
we have $Q \succeq \B^\pi Q \succeq (\B^\pi)^2 Q \succeq ... \succeq (\B^\pi)^{\infty} Q = Q^\pi$, which yields a sequence of increasingly tight upper bounds of $Q^\pi$. 
We leverage a similar idea to prove the monotonicity of our proposed algorithm. 

We are ready to present our main result of convergence, 
in which we show that update \eqref{equ:algupdate} monotonically improves the bound and converges to the optimal solution of \eqref{eqn:upper_bound_opt_framework}, 
if $\F$ is upper-self-contained and $\{\overline{q}_{t,i}\}$ is initialized properly such that they decrease during the first iteration. 
\begin{thm}
\label{thm:monotonic_general}
Assume $\F$ is upper-self-contained function set whose corresponding upper envelope operator is contractive and our evaluate distribution $\mu_{0,\pi}(s,a) = \mu_0(s) \cdot \pi(a|s)$ is full support on $\Sset\times \Aset$, e.g. $\mu_{0,\pi}(x) > 0,~\forall x$.
If we initialize the updates in \eqref{equ:algupdate} with $\overline{q}_{0,i}$, 
such that 
\begin{equation}\label{equ:q1q0}
    \overline{q}_{0,i} \geq \overline{q}_{1,i},~\forall i\in [n], 
\end{equation}
then we have
\begin{equation}
    \overline{Q}_{t}\succeq \overline{Q}_{t+1} \succeq \overline{Q}^\pi,~\forall t= 0, 1,2, \ldots,
\end{equation}
where $\overline{Q}^\pi = \arg\max_{Q\in \F}\{R[Q],~\text{s.t. } Q(x_i) \leq \B^\pi Q(x_i),~\forall i\in[n]\}$.
Therefore, 
$$
R[\overline{Q}_{t}] \geq R[\overline{Q}_{t+1}]  \geq  \overline{R^\pi_{\F}}\geq R^\pi, 
$$ 
and $\lim_{t\to \infty}R[\overline{Q}_{t}] = \overline{R^\pi_{\F}}$. 
\end{thm}
    \begin{proof}
    We focus on upper bound case, lower bound proofs follow similarly.
    Since $\mu_{0,\pi}$ is full support, from lemma \ref{lem:upper-contained} we can see that $\overline{Q}_{t+1}$ is the upper envelope function of data points $(x_i, \overline{q}_{t+1,i})$.
    
    Now we prove the theorem by induction on $t$ for statement $\overline{Q}_t \succeq \overline{Q}_{t+1}$. 
    \begin{enumerate}
        \item Base case. $t = 0$.
        From Lemma \ref{lem:point2function} we have:
        $$
        \overline{Q}_0 \succeq \overline{Q}_1.
        $$
        \item Induction Step.
        Suppose $\overline{Q}_{t-1}\succeq \overline{Q}_t$.
        Then $\overline{q}_{t,i} = r_i + \gamma \Ppi \overline{Q}_{t-1}(x_i) \geq r_i + \gamma \Ppi \overline{Q}_t(x_i) = \overline{q}_{t+1,i}$,
        from lemma \ref{lem:point2function} we have:
        $$
        \overline{Q}_{t}\succeq \overline{Q}_{t+1}.
        $$
    \end{enumerate}
    
    From the induction proof we can see that $\{\overline{Q}_{t}(x)\}$ is a Cauchy sequence with a lower bound for every data point $x$, we know it will finally converge to a function we denote as $\overline{Q}_{\infty}$.
    
    $\overline{Q}_{\infty}\in \F$ will satisfy the constraints $\overline{Q}_{\infty}(x_i) = \B^\pi \overline{Q}_{\infty}(x_i),~\forall i\in [n]$.
    
    On the other hand, from Lemma \ref{lem:inequality2equality} we know that it is the $\overline{Q^\pi_{\F}} = Q^\infty$ almost everywhere.
    
    This leads to a monotone sequence of measures $\{R[\overline{Q}_t]\}_{t=0}^\infty$ with a limit $\overline{\Rpi_\F}$.
    \end{proof}
    
A parallel result holds for the lower bound update \eqref{eqn:lips_update_lower}, 
except that the initialization condition should be $\underline{q}_{0,i}\leq\underline{q}_{1,i}$. See Appendix \ref{sec:lower_bound} for details.

\noindent\textbf{Application to Lipschitz Function Space} General convergence result can be easily applied to the case of Lipschitz functions. We show that the Lipschitz ball $\F_\eta$ satisfies the 
\textit{upper-self-contained} condition, 
and provide a simple initialization method to ensure condition \ref{equ:q1q0}. 
In addition, we establish a fast convergence rate for our algorithm. 

\begin{lem}\label{lem:Lips_superior_init}
i) The Lipschitz ball   $\F_\eta= \{f: L_d(f)\leq \eta\}$ is \textit{upper-self-contained} whose envelope operators are contractive. 

ii) %
Following  the update in \eqref{eqn:lips_update} starting from 
\begin{align}
    \overline{q}_{0,i} = \frac{1}{1-\gamma}\left(r_i + \gamma \eta\E_{x_i'\sim \T^\pi(\cdot|x_i)}[d(x_i, x_i')]\right), 
\end{align}
we have $\overline{q}_{0,i}\geq \overline{q}_{1,i}$ for $\forall i \in [n]$. 

Similarly, for the lower bound 
initializing the update in \eqref{eqn:lips_update_lower} with %
\begin{align*}
    \overline{q}_{0,i} = \frac{1}{1-\gamma}\left(r_i - \gamma \eta\E_{x_i'\sim \T^\pi(\cdot|x_i)}[d(x_i, x_i')]\right), 
\end{align*}
ensures $\underline{q}_{0,i}\leq \underline{q}_{1,i}$ for $\forall i \in [n]$. 

Therefore, the results in theorem~\ref{thm:monotonic} hold. 
\end{lem}
    \begin{proof}
    i) 
    Consider $g(x) = \max_{f\in \F_{\eta}}\{f(x): f(x_i) \leq f_i\}$ for given data points $(x_i, f_i)$, we can see that:
    $$
    g(x) = \min_{i\in [n]}\{ f_i + \eta d(x,x_i) \}
    $$
    we can see that $g(x)$ is $\eta-$Lipschitz continuous.
    
    For the contraction property, from Proposition \ref{pro:lips_update} the upper envelope operator can be written in the following way:
    $$
    \upenv_{\F_{\eta}}(\{x_i, f_i\}_{i\in [n]})(x) = \min_{i\in [n]}\{f_i + \eta d(x,x_i)\},
    $$
    then we have:
    \begin{align*}
        &\upenv_{\F_{\eta}}(\{x_i, f_i\}_{i\in [n]})(x) - \upenv_{\F_{\eta}}(\{x_i, g_i\}_{i\in [n]})(x) \\
        =& \min_{i\in [n]}\{f_i + \eta d(x,x_i)\} - \min_{i\in [n]}\{g_i + \eta d(x,x_i)\} \\
        \leq& \max_{i\in [n]} \{f_i - g_i\},
    \end{align*}
    which implies contraction.
    
    ii) For $\overline{q}_{0,i} = \frac{1}{1-\gamma}\left(r_i + \gamma \eta\E_{x_i'\sim \T^\pi(\cdot|x_i)}[d(x_i, x_i')]\right)$, we have:
    \begin{align*}
        \overline{q}_{0,i} =&  \gamma \overline{q}_{0,i} + r_i + \gamma \eta\E_{x_i'\sim \T^\pi(\cdot|x_i)}[d(x_i, x_i')] \\
        =& r_i + \gamma \E_{x_i'\sim \T^\pi(\cdot|x_i)}[\overline{q}_{0,i} + \eta d(x_i, x_i')] \\
        \geq& r_i + \gamma \E_{x_i'\sim \T^\pi(\cdot|x_i)}[\min_{j\in [n]}\{\overline{q}_{0,j} + \eta d(x_j, x_i')\}] \\
        =& \overline{q}_{1,i}.
    \end{align*}
    
    Similarly we have $\underline{q}_{0,i}\leq \underline{q}_{1,i}$.
    \end{proof}

From Theorem \ref{thm:monotonic_general} and Lemma \ref{lem:Lips_superior_init} we can immediate get Theorem \ref{thm:monotonic}. 
Similarly we can prove the lower bound case.

\subsection{Linear Convergence}
We are ready to prove the linear convergence result for Proposition \ref{thm:error_bound} which follow the similar idea of proving linear convergence of value iteration.

\paragraph{Proposition \ref{pro:convergence_rate}}
    {\itshape
    Following the updates in \eqref{eqn:lips_update}, we have 
    $$
     \zt{R_{\mu_{0,\pi}}}[\overline{Q}_t] - \overline{\Rpi_\F} \leq C \frac{\gamma^t}{1-\gamma},
    $$
    with constant $C := \max_{i\in [n]}|\overline{q}_{1,i} - \overline{q}_{0,i}|$. 
    }
    \begin{proof}
    Consider $\|\overline{q}_{t+1} - \overline{q}_t\|_{\infty}$, where $\overline{q}_t = [\overline{q}_{t,1},\ldots, \overline{q}_{t,n}]^\top$, we have:
    \begin{align*}
        \|\overline{q}_{t+1} - \overline{q}_t\|_{\infty} =& \max_{i} \{\overline{q}_{t+1,i} - \overline{q}_{t,i}\} \\
        =& \gamma \max_{i} \E_{x_i'\sim \T^\pi(\cdot|x_i)}\left[\min_{j\in [n]}\{\overline{q}_{t,j} + \eta d(x_i',x_j)\} - \min_{k\in [n]}\{\overline{q}_{t-1,k} + \eta d(x_i',x_k)\}\right] \\
        \leq& \gamma \max_{i} \E_{x_i'\sim \T^\pi(\cdot|x_i)}\left[\max_{j\in [n]}\{\overline{q}_{t,j} - \overline{q}_{t-1,j} + (\eta d(x_i',x_j) - \eta d(x_i',x_j))\}\right] \\
        =& \gamma \|\overline{q}_t - \overline{q}_{t-1}\|_{\infty}.
    \end{align*}
    
    Therefore we have $\|\overline{q}_{t+1} - \overline{q}_t\|_{\infty} \leq \gamma^t \|\overline{q}_1 - \overline{q}_0\|_{\infty}$, this leads to:
    \begin{align*}
        \|\overline{q}_t - \overline{q}\|_{\infty} =& \|\overline{q}_t - \overline{q}_{\infty}\|_{\infty} \\
        \leq& \sum_{i=t}^{\infty} \|\overline{q}_{i+1} - \overline{q}_i\|_{\infty} \\
        \leq& \sum_{i=t}^{\infty} \gamma^i \|\overline{q}_1 - \overline{q}_0\|_{\infty} \\
        =& \frac{\gamma^t}{1-\gamma} \|\overline{q}_1 - \overline{q}_0\|_{\infty}.
    \end{align*}
    
    To conclude we have:
    \begin{align*}
        \zt{R_{\mu_{0,\pi}}}[\overline{Q}_t] - \overline{\Rpi_\F} =& \E_{x\sim \mu_{0,\pi}}[\overline{Q}_t(x) - \overline{Q}^\pi(x)] \\
        =& \E_{x\sim \mu_{0,\pi}}[\min_{j\in [n]}\{\overline{q}_{t,j} + \eta d(x,x_j)\} - \min_{k\in [n]}\{\overline{q}_k + \eta d(x,x_k)\}] \\
        \leq& \|\overline{q}_t - \overline{q}\|_{\infty} \\
        =& \frac{\gamma^t}{1-\gamma} \|\overline{q}_1 - \overline{q}_0\|_{\infty}.
    \end{align*}
    \end{proof}

\subsection{Tightness of Lipschitz-Based Bounds}

\paragraph{Theorem \ref{thm:error_bound}}
    {\itshape
    Let $\F = \F_{\eta}$ be Lipschitz function class with Lipschitz constant $\eta$.
    Suppose $\X = \Sset \times \Aset$ is a compact and bounded domain equipped with a distance $d\colon \X\times \X \to \RR$. 
    For a set of data points 
    $X = \{x_i\}_{i=1}^n$, we define its covering radius to be 
    $$\varepsilon_{X} = \sup_{x\in \X} \min_{i} d(x,x_i).$$
    We have 
    $$
    \overline{\Rpi_{\F}} - \underline{\Rpi_{\F}} \leq \frac{2\eta}{1-\gamma} \varepsilon_X,
    $$
    where $\eta$ is the Lipschitz constant and $\gamma$ the discount factor. 
    }
    
    \begin{proof}
    Consider $\|\overline{q} - \underline{q}\|_{\infty}$, where $\overline{q} = [\overline{q}_1, \overline{q}_2, ..., \overline{q}_n]^\top$ and $\underline{q} = [\underline{q}_1, \underline{q}_2, ..., \underline{q}_n]^\top$ as the vector of upper and lower functions $\overline{Q}, \underline{Q}$ at point $x_1, x_2,...,x_n$.
    
    Since $r_i = q_i - \gamma \P^\pi Q(x_i)$ for all $Q$ satisfies the finite points Bellman equation (see more details in Lemma \ref{lem:inequality2equality}), we have:
    \begin{align*}
        \overline{q}_i - \underline{q}_i =& \gamma \P^\pi (\overline{Q}(x_i) - \underline{Q}(x_i)) \\
        =& \gamma \E_{x_i'\sim \T^\pi(\cdot|x_i)}\left[\overline{Q}(x_i') - \underline{Q}(x_i') \right] \\
        \leq& \gamma \E_{x_i'\sim \T^\pi(\cdot|x_i)}\left[\min_{j}\{\overline{q}_j - \underline{q}_j + 2\eta d(x_i',x_j)\} \right] \\
        \leq& \gamma \|\overline{q} - \underline{q}\|_{\infty} + 2\eta \gamma \E_{x_i'\sim \T^\pi(\cdot|x_i)}\left[\min_{j} d(x_i',x_j)\right].
    \end{align*}
    
    Let $\varepsilon_n = \sup_{x\in \X} \min_i \{d(x,x_i)\} $, which is the epsilon ball radius given centers $\{x_i\}_{i=1}^n$, typically $\varepsilon \approx O(n^{-\frac{1}{d}})$.
    We have:
    \begin{align*}
        \|\overline{q} - \underline{q}\|_{\infty} =& \max_{i} |\overline{q}_i - \underline{q}_i| \\
        \leq&  \gamma \|\overline{q} - \underline{q}\|_{\infty} + 2\eta \gamma \max_{i} \E_{x_i'\sim \T^\pi(\cdot|x_i)}\left[\min_{j} d(x_i',x_j)\right] \\
        \leq& \gamma \|\overline{q} - \underline{q}\|_{\infty} + 2\eta \gamma \sup_{x\in \X} \min_j d(x, x_j) \\
        =& \gamma \|\overline{q} - \underline{q}\|_{\infty} + 2\eta \gamma \varepsilon_n.
    \end{align*}
    Thus we have:
    \begin{equation*}
        \|\overline{q} - \underline{q}\|_{\infty} \leq \frac{2\eta \gamma}{1-\gamma} \varepsilon_n
    \end{equation*}
    
    Consider $\|\overline{Q} - \underline{Q}\|_{\infty}$, we have:
    \begin{align*}
        \|\overline{Q} - \underline{Q}\|_{\infty} =& \sup_{x\in \X} |\overline{Q}(x) - \underline{Q}(x)|\\
        \leq& \sup_{x\in \X}|\min_{j}\{\overline{q}_j - \underline{q}_j + 2\eta d(x,x_j)\}| \\
        \leq & \|\overline{q} - \underline{q}\|_{\infty} + 2\eta \sup_{x\in \X} \min_j \{d(x,x_j)\} \\
        \leq & \frac{2\eta \gamma}{1-\gamma} \varepsilon_n + 2\eta \varepsilon_n \\
        =& \frac{2\eta }{1-\gamma} \varepsilon_n.
    \end{align*}
    Therefore we have:
    \begin{align*}
        \overline{\Rpi_{\F}} - \underline{\Rpi_{\F}} =& \E_{x\sim \mu_{0,\pi}}[\overline{Q}(x) - \underline{Q}(x)] \\
        \leq& \frac{2\eta}{1-\gamma} \varepsilon_n
    \end{align*}
    \end{proof}

\subsection{Additional Proofs}
\paragraph{Proposition \ref{pro:monotonic_subsampling}}
    {\itshape
    Follow equation \eqref{eqn:lips_update_subsample} with initialization follows \eqref{equ:defq0}, let $\overline{Q}_t$ to be the upper envelope function of data points $\{x_i, \overline{q}_{t,i}\}_{i=1}^n$, we have a similar monotonic result as theorem~\ref{thm:monotonic}:
    \begin{equation*}
        \overline{Q}_{t}\succeq \overline{Q}_{t+1} \succeq \overline{Q}^\pi,~\forall t= 0, 1,2, \ldots,
    \end{equation*}}
    \begin{proof}
    The proof is the same as Theorem \ref{thm:monotonic_general}, once we see that we gradually $\overline{q_i}$ but they will always stay ahead the upper bound $\overline{Q^\pi}(x_i)$, but the convergence is the same as the original algorithm.
    \end{proof}

\paragraph{Proof of Theorem \ref{thm:Q_lips}}
    {\itshape
    Let $\langle \Sset\times\Aset, d_x\rangle$ be a metric space for state action pair $x$ and $\langle \Sset, d_s\rangle$ be a metric space for state $s$.
    Suppose $d_x$ is separable so that $d_x(x_1, x_2) = d_s(s_1, s_2)$ if $a_1 = a_2$.
    If the reward function $r$ and the transition $\T$ are both Lipschitz in the sense that
    \begin{align*}
        &r(x_1) - r(x_2) \leq \|r\|_{\mathrm{Lip}} d_x(x_1,x_2)\\
        &d_s(\T(x_1), \T(x_2)) \leq \|\T\|_{\mathrm{Lip}} d_x(x_1, x_2),~\forall x_1, x_2. 
    \end{align*}
    We can prove that if $\gamma \|\T\|_{\mathrm{Lip}} < 1$, we have
    \begin{equation}
        \|Q^\pi\|_{\mathrm{Lip}}\leq  \frac{\|r\|_{\mathrm{Lip}}}{1-\gamma \|\T\|_{\mathrm{Lip}}},
    \end{equation}
    when $\pi$ is a constant policy. 
    Furthermore, for optimal policy $\pi^*$ with value function $Q^*$, we have:
    \begin{equation}
        \|Q^*\|_{\mathrm{Lip}}\leq  \frac{\|r\|_{\mathrm{Lip}}}{1-\gamma \|\T\|_{\mathrm{Lip}}},
    \end{equation}
    }
    \begin{proof}
    Suppose $a_i = \arg\max_{a} \{Q^\pi(\T^i(x_1), a) - Q^\pi(\T^i(x_2), a)\}$, where $\T^1(x_j) = \T(x_j)$ and $\T^i(x_j) = \T(\T^{i-1}(x), a_{i-1}),~\forall j\in\{1,2\}$ is defined recursively.
    
    If $\pi$ is a constant policy where $\pi(a|s_1) = \pi(a|s_2) = \pi(a),~\forall a, s_1, s_2$, we can actually write $Q^\pi(x) - Q^\pi(x_2)$ as:
    \begin{align*}
        Q^\pi(x_1) - Q^\pi(x_2) =& (r(x_1) - r(x_2)) + \gamma \int_a \left(\pi(a|T(x_1))Q^\pi(\T(x_1),a) - \pi(a|T(x_2))Q^\pi(\T(x_2),a)\right) da \\
        =& (r(x_1) - r(x_2)) + \gamma \int_a \pi(a)\left(Q^\pi(\T(x_1),a) - Q^\pi(\T(x_2),a)\right) da \\
        \leq& (r(x_1) - r(x_2)) + \gamma \max_a \left(Q^\pi(\T(x_1),a) -  Q^\pi(\T(x_2),a)\right) \\
        =& (r(x_1) - r(x_2)) + \gamma \left(Q^\pi(\T(x_1),a_1) -  Q^\pi(\T(x_2),a_1)\right).
    \end{align*}
    
    And similarly we have 
    $$
    Q^\pi(\T^{i-1}(x_1),a_{i-1}) -  Q^\pi(\T^{i-1}(x_2),a_{i-1}) \leq r(\T^{i-1}(x_1),a_{i-1}) - r((\T^{i-1}(x_2),a_{i-1})) + \gamma \left(Q^\pi(\T^{i}(x_1),a_{i}) -  Q^\pi(\T^{i}(x_2),a_{i})\right).
    $$
        
    Therefore we have:
    \begin{align*}
        Q^\pi(x_1) - Q^\pi(x_2) \leq & (r(x_1) - r(x_2)) + \sum_{i=1}^\infty \gamma^i (r(\T^i(x_1),a_i) - r(\T^i(x_2),a_i)) \\
        \leq& \lambda_r d_x(x_1, x_2) + \sum_{i=1}^\infty \gamma^i (r(\T^i(x_1),a_i) - r(\T^i(x_2), a_i)) ~\textcolor{blue}{\text{//according to definition of max operator over }a_i}\\
        \leq& \lambda_r d_x(x_1, x_2) + \sum_{i=1}^\infty \gamma^i \lambda_r d_x( [\T^i(x_1),a_i], [\T^i(x_2),a_i]) ~\textcolor{blue}{\text{//Lipschitz of reward function}}\\
        =& \lambda_r d_x(x_1, x_2) + \sum_{i=1}^\infty \gamma^i \lambda_r d_s( \T^i(x_1), \T^i(x_2)) ~\textcolor{blue}{\text{//by the assumption } d_x(x_1,x_2) = d_s(s_1,s_2)~\text{if }a_1 = a_2}\\
        \leq& \lambda_r \left( d_x(x_1, x_2) + \sum_{i=1}^\infty \gamma^i \lambda_T^i d_x(x_1, x_2) \right)\\
        =& \frac{\lambda_r}{1-\gamma \lambda_T} d_x(x,\bar x).
    \end{align*}
    
    The last inequality can be proved inductively by
    \begin{align*}
        d_s( \T^i(x_1), \T^i(x_2)) \leq& \lambda_T d_x([\T^{i-1}(x_1), a_{i-1}], [\T^{i-1}(x_2), a_{i-1}]) \\
        =& \lambda_T d_s (\T^{i-1}(x_1), \T^{i-1}(x_2) \\
        \leq & ... \\
        \leq & \lambda_T^{i-1} d_s(\T(x_1), \T(x_2)) \\
        \leq & \lambda_T^i d_x(x_1, x_2).
    \end{align*}
    
    For $Q^*$ case we can have the similar derivation where at the beginning of the proof we have:
    \begin{align*}
        Q^*(x_1) - Q^*(x_2) =& (r(x_1) - r(x_2)) + \gamma (\max_a \{Q^\pi(\T(x_1),a)\} - \max_a \{Q^\pi(\T(x_2),a)\}) \\
        \leq& (r(x_1) - r(x_2)) + \gamma \max_a \left(Q^\pi(\T(x_1),a) -  Q^\pi(\T(x_2),a)\right) \\
    \end{align*}
    \end{proof}

\section{More Discussions on Lipschitz Norm}
\label{sec:app-lips}

We show the non-identifiable results of upper bound of Lipschitz norm by constructing a set of possible $Q$ functions that are consistent with the data but with an unbounded increasing Lipschitz norm.

We first show that if our function set provides at least two different solutions, we have a nontrivial solution in null space.
\begin{lem}[Null Space of Finite Bellman Constraints]\label{lem:null-space}
There is a non-zero Lipschitz continual function $G$ with $\|G\|_{d,\mathrm{Lip}}\leq 2\eta$ such that:
$$
G(x_i) = \gamma \Ppi G(x_i),~\forall i\in [n]\,,
$$
once the solution in $\F_{\eta}$ is not unique.
\end{lem}
\begin{proof}
Suppose $Q_1, Q_2\in \F_{\eta}$ satisfies all the finite Bellman constraints.
Consider $G = Q_1 - Q_2$, we have $\|G\|_{d,\mathrm{Lip}}\leq 2\eta$ and
$$
G(x_i) = Q_1(x_i) - Q_2(x_i) = \gamma \Ppi (Q_1 - Q_2)(x_i) = \gamma \Ppi G(x_i).
$$
\end{proof}

Using the nontrivial solution in null space we can construct arbitrarily large Lipschitz norm solution of $Q$ that are consistent with data.
\begin{thm}[Non-identifiable of Upper Bound of Lipschitz Function]
If there is more than one Lipschitz functions satisfies finite Bellman constraints,
For all $\eta$ we can always find $Q$ satisfies Bellman constraint and $\|Q\|_{d,\mathrm{Lip}}\geq \eta$.
\end{thm}
\begin{proof}
By using $G$ in Lemma \ref{lem:null-space} we can construct a set of $Q_{\lambda}$ satisfies finite samples Bellman equation in \eqref{equ:Bellmani}.
$$
Q_\lambda = Q^\pi + \lambda G.
$$
where if we pick $\lambda \geq (\eta + \|Q^\pi\|_{d,\mathrm{Lip}})/\|G\|_{d,\mathrm{Lip}}$ we have:
$$
\|Q_\lambda\|_{d,\mathrm{Lip}} \geq (\eta + \|Q^\pi\|_{d,\mathrm{Lip}}) - \|Q^\pi\|_{d,\mathrm{Lip}} = \eta.
$$
and $Q_\lambda$ satisfies finite samples Bellman constraints:
$$
Q_\lambda(x_i) = Q^\pi(x_i) + \lambda G(x_i) = r_i + \gamma \Ppi (Q^\pi(x_i) + \lambda G(x_i)) = \B^\pi Q_\lambda(x_i).
$$
\end{proof}

\section{Experimental Settings}
\label{sec:experiment_details}

\paragraph{Comparison Results using Thomas Bounds\cite{thomas2015high}}
We compare our method with lower bound estimation from \citet{thomas2015high}, whose bound leverage a sophisticated concentration bound by importance sampling estimators.
Their method is based on the unbiased \textit{importance sampling} \citep{precup00eligibility} estimator of $\R^\pi$:
$$
    \hat{R}^{\pi}_{\rm IS} := \underbrace{R(\tau)}_{\text{return}}\underbrace{\prod_{t=1}^{T}\frac{\pi(a_t|s_t)}{\pi_{0}(a_t|s_t)}}_{\text{importance weight}}\,,
$$
where $\tau$ is the trajectory and $R(\tau)$ is the normalized and discounted average reward of a trajectory.
They obtain a high confidence lower bounds for $\hat{R}^{\pi}_{\rm IS}$ by leveraging the concentration inequality with an adjust threshold parameter specified by user.

\begin{table}[ht]
    \centering
    \renewcommand{\arraystretch}{1.2}
\scalebox{0.90}{
\begin{tabular}{c|c|c|c|c|c|c}
    \hline
    Number of Trajectories & 2 & 4 & 6 & 10 & 20 & 30 \\
    \hline
    Thomas Relative Lower Bound & \textbf{-8.61e-03} & -2.87e-03 & -1.72e-03 & -9.56e-04 & -4.53e-04 & -2.97e-04\\
    \hline
    Our Relative Lower Bound & -0.131 & \textbf{0.343} & \textbf{0.536} & \textbf{0.698} &  \textbf{0.805} & \textbf{0.850}\\
    \hline
\end{tabular}
}
    \caption{Comparison with Thomas Lower Bound \citet{thomas2015high}, close to 1 is better.}
    \label{tab:my_label}
\end{table}

We empirically evaluate the method on Pendulum environment, where we use the same default settings as we conduct experiments for our algorithms.
Table \ref{tab:my_label} shows the results compared to our lower bound.
We pick the best result choosing the threshold number from $\{0.001, 0.01, 0.1, 1.0, 10.0\}$ and we set the confidence level of Thomas lower bound to be $95\%$.
All the numbers are the relative reward that divided by the ground truth.

We can see that Thomas' lower bound is not sensitive to small number of samples, which is almost near 0.
This is mainly because the importance ratio between the target policy $\pi(a|s)$ and the behavior policy $\pi_0(a|s)$ for each trajectory sample goes to 0 due to the curse of horizon (we use horizon length = 100 here) ,
which makes IS based estimator not a proper method for long (or infinite) horizon problems.
As a consequence, concentration confidence bounds based on IS estimators could be potentially loose in such problems (The number of trajectories used in \citet{thomas2015high} can be $n = 10^{7}$ if they want to get a tight lower bound).

\paragraph{Synthetic Environment with A Known Value Function}
The transition of this environment is a one dimension linear function:
$$
\T(s,a) = 0.8 s - 0.4 a - 0.1,
$$
and the target policy we use is
$$
\pi(s,\xi) = 1.5 s - 0.1 + \xi,~\xi \sim \N(0,\sigma^2),
$$
where Gaussian variance $\sigma = e^{-5}$.
And the historical data is pre-collected by a behavior policy $\pi_0$ similar to $\pi$ but with a larger variance.

\begin{figure}
    \centering
    \begin{tabular}{cccc}
        \includegraphics[width = 0.23\textwidth]{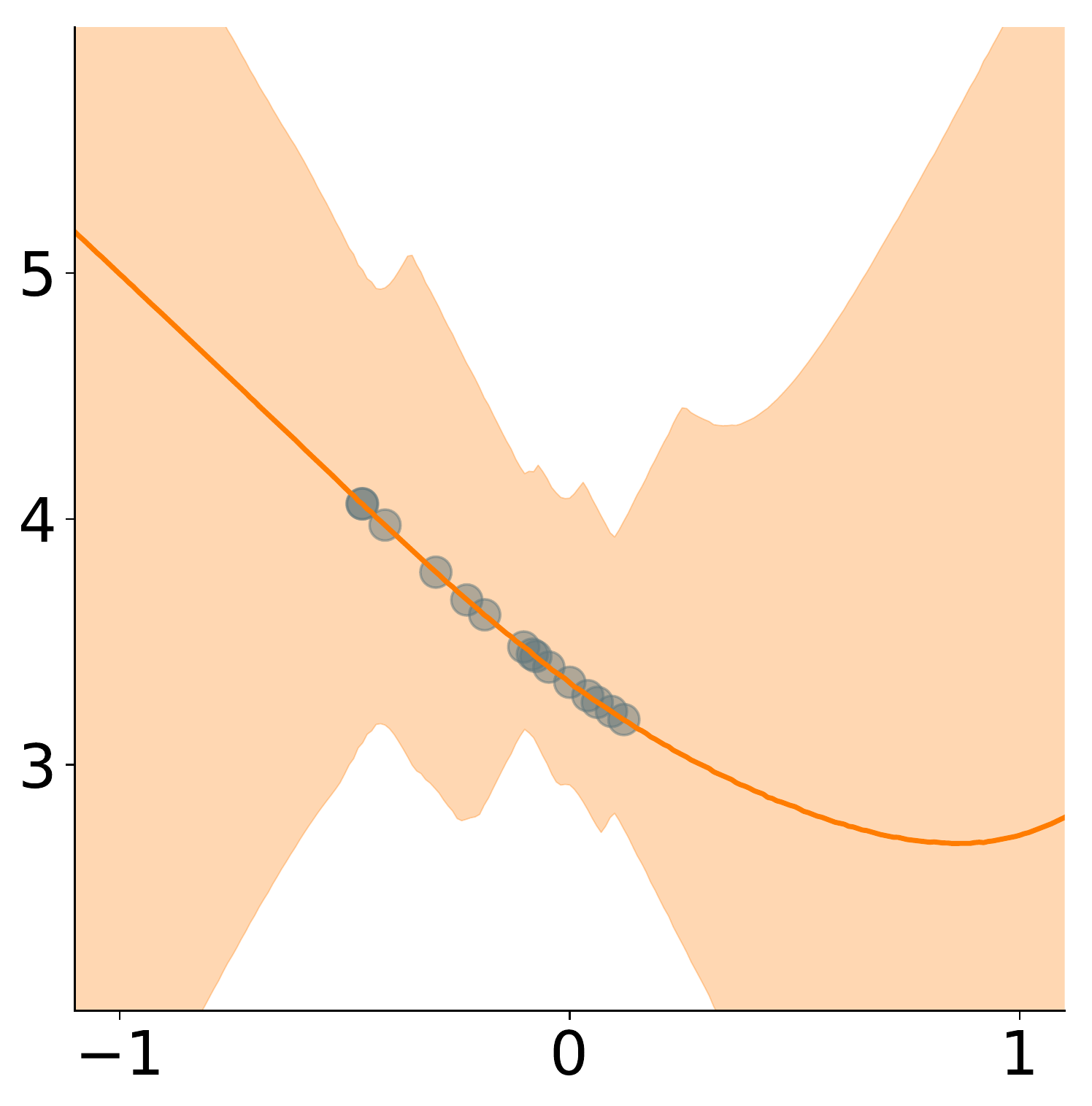}&
        \includegraphics[width = 0.23\textwidth]{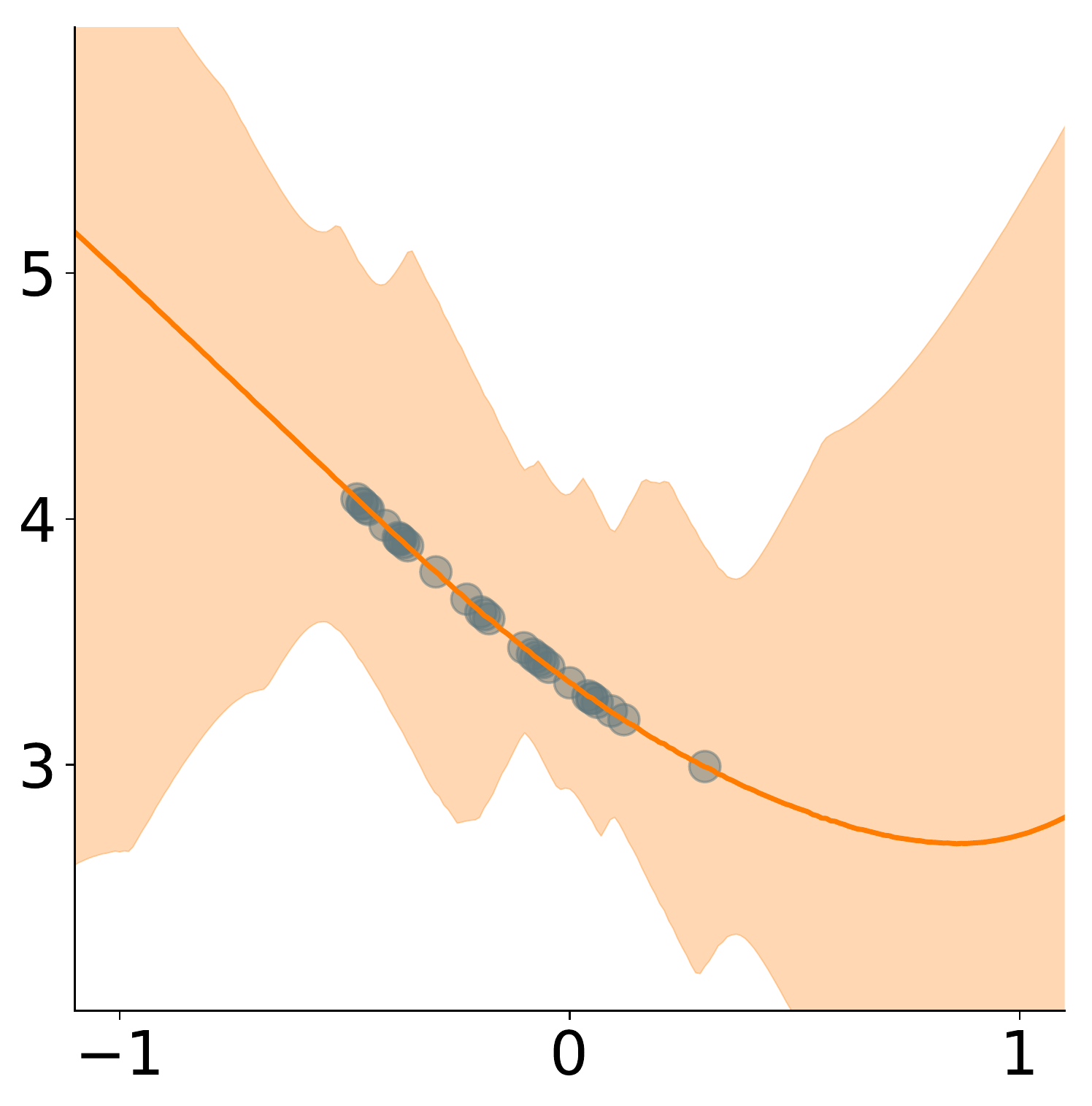}&
        \includegraphics[width = 0.23\textwidth]{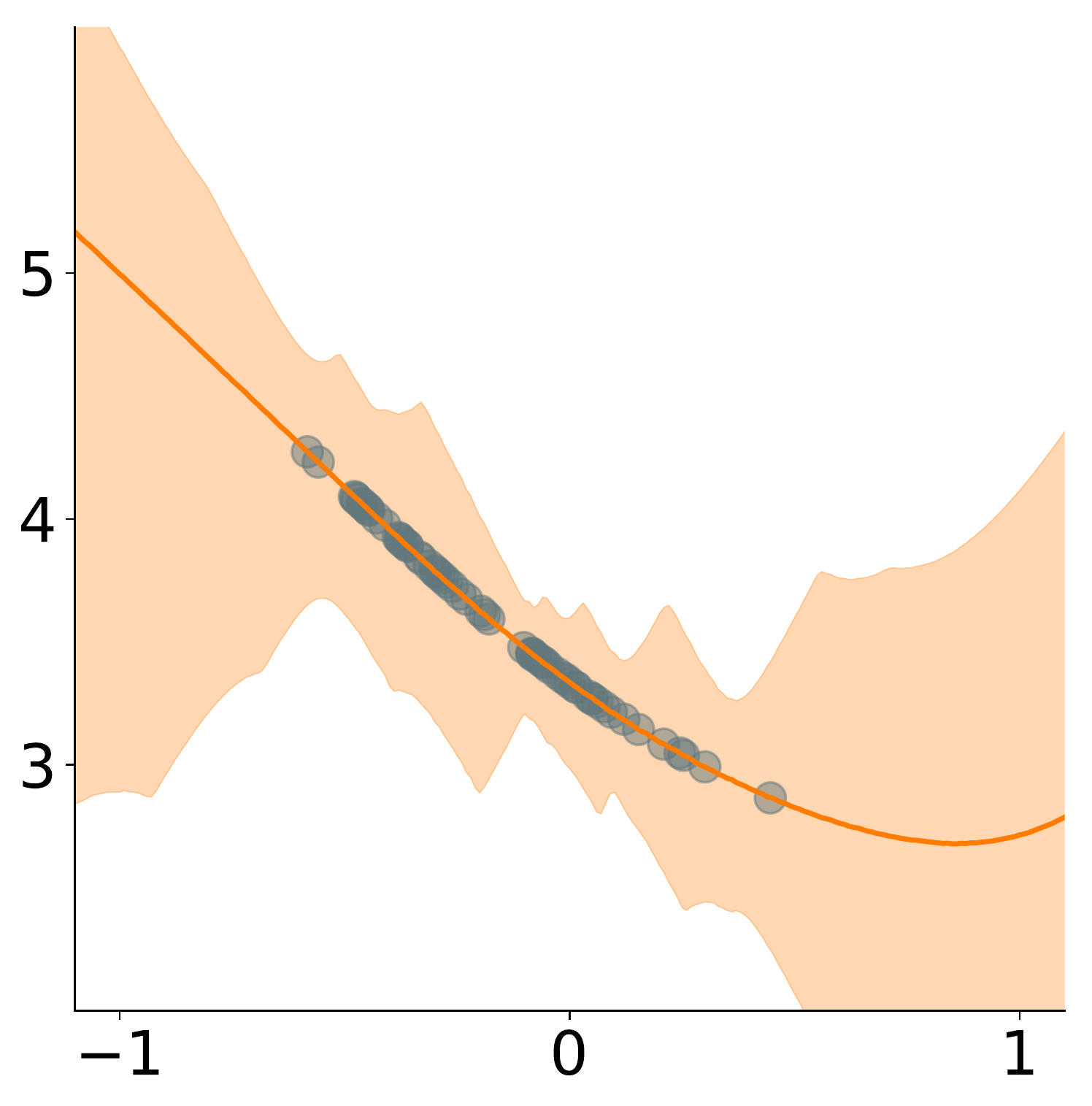}&
        \includegraphics[width = 0.23\textwidth]{figure/toy/value_demonstration_toy_size=100.pdf}\\
        $n = 15$ & $n=30$ & $n=60$ & $n=100$
    \end{tabular}
    \caption{Landscape of upper and lower bound of $V^\pi$ with different number of samples.}
    \label{fig:demonstration}
\end{figure}

The predefined q-function is $Q^\pi(s,a) = f(s) + f(a-\frac{\pi}{2})$ where $f(x) = \sqrt{x^2 + x \sin(x) + 1}$.
For distance metric we use Euclidean distance $d(x,y) = \|x-y\|_2$.
Under this distance metric, we can calculate the exact Lipschitz constant:
\begin{align*}
    L_d(Q^\pi) =& \sup_{s,a} \lim_{\epsilon\to 0} \frac{Q^\pi(s,a) - Q^\pi(s + \frac{\sqrt{2}}{2}\epsilon, a + \frac{\sqrt{2}}{2}\epsilon))}{\epsilon} \\
    =& 2\sup_{s} \lim_{\epsilon \to 0} \frac{f(s) - f(s + \frac{\sqrt{2}}{2}\epsilon)}{\epsilon} \\
    =& 2
\end{align*}

Figure \ref{fig:demonstration} shows a full landscape of upper and lower bounds of state value function $V^\pi$ under different number of samples, similar to Figure \ref{fig:RL_toy}(e).

\paragraph{Pendulum Environment}
We learn a feature map $\Phi$ for $Q^\pi$ by a two hidden layers neural network $[f_1, f_2, f_3]$, where the input layer is state action pair $x_0 = x$, the first hidden layer is $x_1 = f_1(x) = \text{RELU}(W_1^\top x + b_1)$ with $100$ hidden dimension matrix $W_1, b_1$.
We set $x_2 = f_2(x_0, x_1) = [x_0, \text{RELU}(W_2^\top x_1 + b_2)]^\top$ as feature layer, where we let the concatenate the input layer and a relu of linear layer for the hidden layer as our feature.
We add the input layer to our feature layer to ensure that our distance function $d(x_0, \bar x_0) = \|x_2 - \bar x_2\|$ is a true distance function (not a semi-distance one because $x_2 = \bar x_2$ requires $x = \bar x$).
And finally the last layer is a linear layer with output dimension $1$ as the output of q-function.
Thus our approximate q-function can be represented as
$$
Q(x) = W_3^\top \Phi(x) + b_3,
$$
where $\Phi(x) = f_2(x, f_1(x))$.

We apply fitted value evaluation algorithm \cite{munos2008finite,le2019batch} to use off-policy data to pre-train a $Q^\pi$, then we use the feature layer $x_2 = \Phi(x)$ as our feature, and set the Lipschitz constant approximately 2 times the $\ell_2$ norm of the last linear layer parameter $W_3$ as our default parameter.

\paragraph{HIV simulator}
We follow exactly the same settings as \citet{liu2018representation}.

\end{document}